\documentclass{article}
\usepackage[final]{corl_2022}

\usepackage{xcolor}
\usepackage{etoolbox}
\usepackage{upgreek}
\usepackage{amsmath}
\usepackage{comment}

\usepackage{amssymb}
\usepackage{amsthm}

\usepackage{caption}
\usepackage{subcaption}
\usepackage{graphicx}
\usepackage{bm}
\usepackage{bbm}

\newcommand{\pol}{\pi}
\newcommand{\Pol}{\Pi}

\newtheorem{lemma}{Lemma}
\newtheorem{remark}{Remark}
\newtheorem{prop}{Proposition}

\DeclareMathOperator*{\argmin}{arg\,min}

\newcommand{\R}{\mathbb{R}}

\newcommand{\N}{\mathbb{N}}

\newcommand{\fcnote}[1]%
    {\textcolor{orange}{\textbf{FC: #1}}}
\newcommand{\twnote}[1]%
    {\textcolor{cyan}{\textbf{TW: #1}}}
\newcommand{\aanote}[1]%
    {\textcolor{blue}{\textbf{AA: #1}}}
\newcommand{\ksnote}[1]%
    {\textcolor{red}{\textbf{KS: #1}}}

\newlength\tindent
\let\sp=\sqparen

\newcommand{\msnote}[1]%
    {\textcolor{cyan}{\textbf{MS: #1}}}

\newtheorem{assumption}{Assumption}
\newtheorem{definition}{Definition}

\newcommand{\KL}{\mathcal{K}{\mathcal{L}}}
\newcommand{\K}{\mathcal{K}}

\newtheorem{theorem}{Theorem}

\newcommand{\U}{\mathcal{U}}
\newcommand{\X}{\mathcal{X}}

\newcommand{\lyap}{W}

\title{Lyapunov Design for Robust and Efficient Robotic Reinforcement Learning}

\author{
  Tyler Westenbroek$^{1,*}$\\
  \texttt{westenbroekt@berkeley.edu} \\
  \And
 Fernando Casta\~neda$^{2,*}$ \\
  \texttt{fcastaneda@berkeley.edu}
  \And
  Ayush Agrawal$^{2,*}$\\
  \texttt{ayush.agrawal@berkeley.edu} \\
  \And
  Shankar Sastry$^{1}$\\
  \texttt{sastry@coe.berkeley.edu}
  \And
  Koushil Sreenath$^2$\\
  \texttt{koushils@berkeley.edu}
  \\
  $^1$Department of Electrical Engineering and Computer Sciences, UC Berkeley \\
  $^2$Department of Mechanical Engineering, UC Berkeley \\
  $^*$ Equal Contribution
  \thanks{This work was supported in part by NSF Grants CMMI-1944722 and CMMI-1931853, LOGiCS (Learning-Driven Oracle-Guided Compositional Symbiotic Design of Cyber-Physical Systems), and Defense Advanced Research Projects Agency award number FA8750-20-C-0156. The work of Fernando Casta\~neda was partially supported through a fellowship from Fundación Rafael del Pino, Spain.}
}
\begin{document}

\maketitle

\vspace{-2em}
\begin{abstract}
    Recent advances in the reinforcement learning (RL) literature have enabled roboticists to automatically train complex policies in simulated environments. However, due to the poor sample complexity of these methods, solving RL problems using real-world data remains a challenging problem. This paper introduces a novel cost-shaping method which aims to reduce the number of samples needed to learn a stabilizing controller. The method adds a term involving a Control Lyapunov Function (CLF) -- an `energy-like' function from the model-based control literature -- to typical cost formulations. Theoretical results demonstrate the new costs lead to stabilizing controllers when smaller discount factors are used, which is well-known to reduce sample complexity. Moreover, the addition of the CLF term `robustifies' the search for a stabilizing controller by ensuring that even highly sub-optimal polices will stabilize the system. We demonstrate our approach with two hardware examples where we learn stabilizing controllers for a cartpole and an A1 quadruped with only seconds and a few minutes of fine-tuning data, respectively. Furthermore, simulation benchmark studies show that obtaining stabilizing policies by optimizing our proposed costs requires orders of magnitude less data compared to standard cost designs.
\end{abstract}

\begin{figure}[h!]
\vspace{-1em}
     \centering
     \begin{subfigure}[b]{0.92\textwidth}
         \centering
         \includegraphics[width=\textwidth]{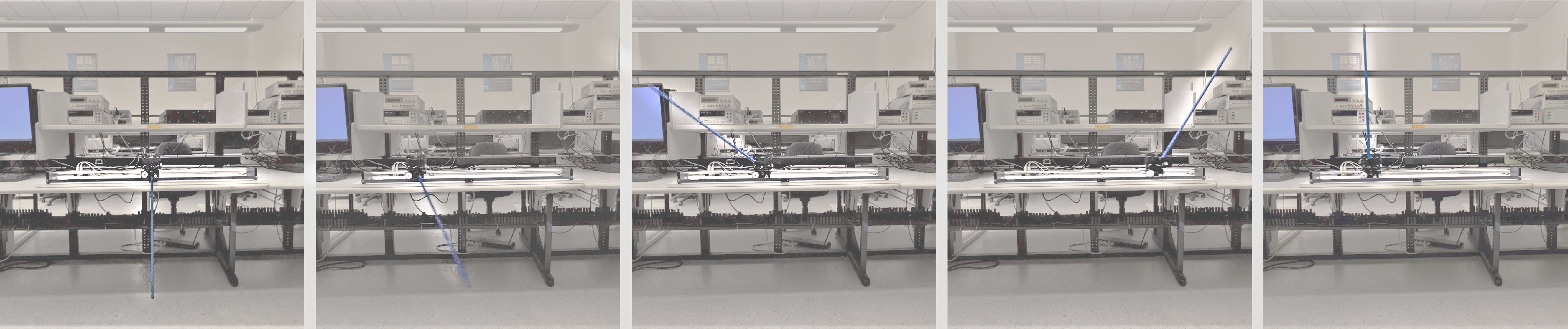}
         \label{fig:cartpole_cover}
         \vspace{-10pt}
     \end{subfigure}
     \begin{subfigure}[b]{0.92\textwidth}
         \centering
         \includegraphics[width=\textwidth]{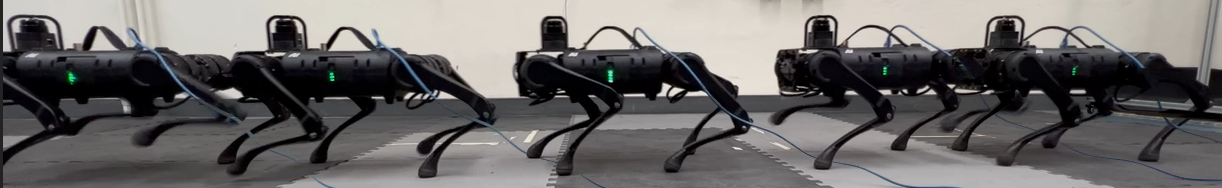}
         \label{fig:a1_cover}
         
     \end{subfigure}
     \vspace{-1em}
        \caption{\small We learn precise stabilizing policies on hardware for the \texttt{Quanser} cartpole  \cite{quanser_products_2021} (top) and the \texttt{Unitree A1} quadruped \cite{unitree} (bottom) using only seconds and a few minutes of real-world data, respectively. A video of our experiments can be found here \url{https://youtu.be/l7kBfitE5n8} }
        \label{fig:cover}
        \vspace{-1em}
\end{figure}
\vspace{-0.5em}
\section{Introduction}
A key challenge in robotics is reasoning about the long-horizon behavior induced by a control policy. This is because important system properties such as stability are inherently long-horizon phenomena. In reinforcement learning (RL), the \emph{discount factor} implicitly controls how far into the future policy optimization algorithms plan when optimizing the objective specified by the user. Standard approaches to designing objective functions for robotic RL, such as penalizing the distance to a reference trajectory, inherently require a large discount factor to learn control policies which stabilize the system \cite{postoyan2016stability,gaitsgory2015stabilization}. Unfortunately, problems with large discount factors can be extremely difficult to solve, often requiring vast data sets and careful tuning of hyper-parameters \cite{franccois2015discount}. As a number of recent success stories have demonstrated \cite{lee2020learning, kumar2021rma, peng2020learning, peng2018sim, li2021reinforcement, belkhale2021model}, ever-increasing computational resources can be used to solve these problems in simulation and deploy the resulting controllers directly on the real-world system. However, because it is impractical to model every detail of complex hardware platforms, achieving the best performance will require learning from real-world data.

This paper introduces a cost-shaping framework which enables users to reliably learn stabilizing control policies with small amounts of real-world data by solving problems with small discount factors. Our approach uses \emph{Control Lyapunov Functions} (CLFs), a standard design tool from the control theory literature \cite{artstein1983stabilization, SONTAG1989117, ames_clf_manipulation,ames2014rapidly}. CLFs are `energy-like' functions for the system which reduce the search for a stabilizing controller to a myopic one-step criterion. In particular, any controller which decreases the energy of the CLF at each instance of time will stabilize the system. Thus, CLFs reduce the long-horizon objective of stabilizing the system to a simple one-step condition. When a CLF is available and the dynamics are known, constructive techniques from the control literature can be used to synthesize a stabilizing controller. However, when there is uncertainty in the dynamics, it is difficult to guarantee that a controller will always decrease the value of the CLF, or that we have even designed a true CLF for the system. 

Our approach is to $1)$ design an approximate
CLF for the real-world system using an approximate dynamics model and $2)$ modify the `standard' choice of cost functions mentioned above by adding a term which incentivizes controllers which decrease the approximate CLF over time. This technique effectively uses the approximate CLF as supervision for reinforcement learning, enabling the user to embed known system structures into the learning process while retaining the flexibility of RL to overcome unknown dynamics. Indeed, as our analysis demonstrates, when our approach is used reinforcement learning algorithms implicitly learn to `correct' the approximate CLF provided by the user. When the candidate CLF is close to being a true CLF for the system (in a sense we make precise below), a stabilizing controller can be efficiently learned by solving a problem with a small discount factor. Moreover, the addition of the approximate CLF `robustifies' the search for a stabilizing controller by ensuring that even highly suboptimal policies will stabilize the system. Finally, in situations where it is too difficult to design a nominal CLF by hand, we demonstrate how one can be learned using a simulation model and the standard style of RL objective discussed above. Specifically, we use the value function learned by the RL algorithm as an approximate CLF for the real-world system. 
Altogether, beyond accelerating and robustifying RL, our approach also expands the applicability of CLF-based design techniques.

We apply this technique to develop data-efficient fine-tuning strategies, wherein a nominal controller developed using a simulation model is refined with small amounts of real-world data.  For the A1 experiment, the nominal controller is a model-based control architecture \cite{da2021learning}, and we hand-design a CLF using a highly simplified linearized reduced-order model for the system. Even though this model is very crude, we are nonetheless able to learn a precise tracking controller for this 18 DOF system with only 5 minutes of real-world data. 
For the cartpole swing-up task we used the value function from a simulation-based RL problem as the candidate CLF for the real-world system, using the learning process described above. Our fine-tuning approach then learned a robust swing-up controller after observing only one 10 second trajectory from the real-world system.

\vspace{-0.6em}
\subsection{Related Work}
We outline how our approach departs from related work; Appendix \ref{sec:extra_lit} contains further discussion. 
\noindent
\textbf{Discount Factors, Sample Complexity and Reward Shaping:} It is well-understood that the discount factor has a significant effect on the size of the data set that RL algorithms need to achieve a desired level of performance. Specifically, it has been shown in numerous contexts \cite{bertsekas1996neuro, schulman2017proximal, munos2008finite,postoyan2019stability} that smaller discount factors lead to problems which can be solved more efficiently. This has led to a number of works which explicitly treat the discount factor as a parameter which can be used to control the complexity of the problem alongside reward shaping techniques \cite{jiang2015dependence,petrik2008biasing,franccois2015discount,tessler2020reward,cheng2021heuristic,ng1999policy}. Compared to these works, our primary contribution is to demonstrate how CLFs can be combined with model-free algorithms to rapidly learn stabilizing controllers for robotic systems.

\noindent
\textbf{Fine-tuning with Real World Data: } Recently, there has been much interest in using RL to fine-tune policies which have been pre-trained in simulation \cite{smith2021legged,julian2020never,julian2020efficient,mandi2022effectiveness}. These methods typically optimize the same cost function with a large discount factor in both simulation and on the real robot. In contrast, using our cost reshaping techniques, we solve a different problem with a smaller discount factor on hardware which can be solved more efficiently.
In Appendix \ref{sec:experiment_details}, we show that our method outperforms typical fine-tuning approaches under moderate perturbations to the dynamics model. 

\noindent
\textbf{Learning with Control Lyapunov Functions:} A number of recent works have also tried to overcome the reality gap using data-driven methods to improve CLF-based controllers \cite{taylor2019episodic,taylor2019control,westenbroek2020learning,westenbroek2021combining,castaneda2020gaussian,choi2020reinforcement}. While these methods work well when a true CLF for the real-world system is available, our method is more general as we can still efficiently learn stabilizing controllers when only an approximate CLF is available by modulating the discount factor used to optimize our cost.

\vspace{-0.5em}



\section{Background and Problem Setting} 
Throughout the paper we will consider deterministic discrete-time systems of the form:
\begin{equation}\label{eq:dynamics}
    x_{k+1} = F(x_k, u_k),
\end{equation}
where $x_k \in \X \subset \R^n$ is the state at time $k$, $u_k \in \U \subset \X$ is the input applied to the system at that time, and $F \colon \X \times \U \to \R^n$ is the transition function for the system. This general nonlinear model is broad enough to cover many important continuous control tasks for robotics. We will let $\Pi$ denote the space of all control polices $\pi \colon \X \to \U$ for the system. To ease exposition, for our theoretical analysis we will focus on the case where the goal is to stabilize the system to a single point, namely the origin. Through our examples we will demonstrate how our cost-shaping technique can be leveraged to achieve more complicated tasks, and in Section \ref{sec:limitations} we outline a path for extending our theoretical results to these settings in future work.

\vspace{-0.5em}
\subsection{Control Lyapunov Functions}
Control Lyapunov Functions \cite{artstein1983stabilization, SONTAG1989117, ames_clf_manipulation,ames2014rapidly}  are `energy-like' functions for the dynamics \eqref{eq:dynamics}:

\begin{definition}
We say that a positive definite function $\lyap \colon \R^n \to \R$ is a \emph{Control Lyapunov Function} (CLF) for \eqref{eq:dynamics} if the following condition holds for each $x \in \X \backslash \{0\}$:  
\begin{equation}\label{eq:clf}
    \min_{u \in \U} \lyap(F(x,u)) - \lyap(x) < 0.
\end{equation}
\end{definition}
The condition \eqref{eq:clf} ensures that for each $x \in \X$ there exists a choice of input which decreases the `energy' $W(x)$. 
Any policy which satisfies the one-step condition $W(F(x,\pi(x))) - W(x) <0$ can be guaranteed to asymptotically stabilize the system \cite{kellett2003results} (see Appendix \ref{sec:stability} for background on stability theory). Given a CLF for the system, model-based methods constructively synthesize a controller which satisfies this property using either closed-form equations \cite{SONTAG1989117} or by solving an online (convex) optimization problem \cite{freeman2008robust,ames2014rapidly} to satisfy \eqref{eq:clf}. However, when the dynamics are unknown it is difficult to ensure that we have synthesized a `true' CLF for the system. 

\begin{remark} (\textit{Designing Control Lyapunov Functions})
While there is no general procedure for designing CLFs by hand for general nonlinear systems, there do exist constructive procedures for designing CLFs for many important classes of robotic systems, such as manipulator arms \cite{ames_clf_manipulation} and robotic walkers \cite{ames2014rapidly} using structural properties of the system. Moreover, in our examples we will investigate how a CLF can be learned from a simulation model and how very coarse CLF candidates can be used to accelerate learning a stabilizing controller. 
\end{remark}
\subsection{Stability of Dynamic Programming and Reinforcement Learning} \label{sec:og_cost}
Here we investigate how a common class of cost functions found in the literature can be used to learn stabilizing controllers. 
In particular, we consider a running cost $\ell \colon \X \times \U \to \R$ of the form $\ell(x,u) = Q(x) + R(u)$, where $Q \colon \mathcal{X} \to \R$ is the state cost and $R\colon \U \to \R$ is the input cost. Both $Q$ and $R$ are assumed to be positive definite (in practice, both are usually quadratic). Given a  policy $\pol \in \Pol$, discount factor $\gamma \in \sp{0,1}$, and initial condition $x_0 \in \X$, the associated long-run cost is: 
\begin{align}\label{eq:cost1}
 V_\gamma^\pol(x_0) =& \sum_{k=0}^{\infty}\gamma^{k} \ell(x_k,\pol(x_k)) \\ \nonumber
    & \text{s.t. }  x_{k+1} = F(x_k,\pol(x_k)),
\end{align}
where $V_\gamma^{\pol} \colon \X \to \R\cup \{\infty\}$ is the \emph{value function} associated to $\pol$. Small discount factors incentivize policies which greedily optimize a small number of time-steps into the future, while larger discount factors promote policies which reduce the cost in the long-run. We say that a policy $\pol_\gamma^* \in \Pol$ is $\emph{optimal}$ if it achieves the smallest cost from each $x \in \X$: 
\begin{equation*}
    V_\gamma^{\pi_\gamma^*}(x) = V_\gamma^*(x):= \inf_{\pol \in \Pol} V_\gamma^\pol(x), \ \ \ \ \forall x \in \X,
\end{equation*}
where $V_\gamma^* \colon \X \to \R \cup \{\infty\}$ is the \emph{optimal value function}. Together $V_\gamma^*$ and $\pi_\gamma^*$ capture the `ideal' behavior induced by the cost function \eqref{eq:cost1}. It is well-known \cite{bertsekas1996neuro} that the optimal value function will satisfy the Bellman equation: 
\begin{equation}\label{eq:bad_bellman}
    V_\gamma^{*}(x) = \inf_{u \in \U} \big[\gamma V_\gamma^*(F(x,u))  + \ell(x,u)\big] ,\ \ \ \ \forall x \in \X,
\end{equation}
and an optimal policy $\pi_\gamma^*$ will satisfy  $\pi_\gamma^*(x) \in \argmin_{u \in \U} \big[ \gamma V_\gamma^*(F(x,u)) + \ell(x,u)\big]$, $\forall x \in \X$. Unfortunately, it is impractical to directly search over $\Pi$ to find a policy which meets these conditions. This necessitates the use of function approximation schemes (e.g. feed-forward neural networks) to instead represent a subset of policies $\hat{\Pol} \subset \Pol$ to search over. Indeed, modern RL approaches for robotics randomly sample the space of trajectories to optimize problems of the form:
\begin{equation}\label{eq:rl1}
 \inf_{\pol \in \hat{\Pol}} \mathbb{E}_{x_0 \sim X_0} \big[V_\gamma^\pi(x_0)\big],
\end{equation}
where $X_0$ is a distribution over initial conditions. While this approach enables these methods to optimize high-dimensional policies, they are data-hungry, can display high-variance and thus frequently return highly sub-optimal policies when data is limited. To better understand the effect that this has on the stability of learned policies, for each $\pol \in \hat{\Pol}$ and $\gamma \in \sp{0,1}$ define the \emph{optimality gap}:
\begin{equation*}
    \epsilon_\gamma^\pol(x) = V_\gamma^\pol(x) - V_\gamma^*(x).
\end{equation*}
The temporal difference equation \cite{bertsekas1996neuro} dictates that for each $x \in \X$ the policy satisfies:
\begin{equation}\label{eq:bellman}
    V_\gamma^\pi(x)  = \gamma V_\gamma^\pi(F(x,\pi(x))) + \ell(x,\pi(x)). 
\end{equation}
From these equations we can obtain:
\begin{align}\label{eq:algebra1}
 V_\gamma^{\pol}(F(x,\pol(x)))  - V_\gamma^\pi(x)& =\frac{1}{\gamma} \big(-\ell(x,\pi(x)) + (1-\gamma)V_\gamma^{\pol}(x) \big)\\
& = \frac{1}{\gamma}\big(-\ell(x,\pi(x)) + (1-\gamma)[V_\gamma^{*}(x) + \epsilon_\gamma^\pi(x)]\big)\\ \label{eq:decay1}
&\leq 
 \frac{1}{\gamma} \big(-Q(x) + (1-\gamma)[V_\gamma^*(x) + \epsilon_\gamma^\pol(x)]\big),
\end{align}
where we have first rearranged \eqref{eq:bellman}, then used $V_\gamma^\pi(x) = V_\gamma^*(x) + \epsilon_\gamma^\pi(x)$, and finally we have used $\ell(x,\pi(x)) \geq Q(x)$. Inequalities of this sort are the building block for proving the stability of suboptimal polices in the dynamic programming literature \cite{gaitsgory2015stabilization,postoyan2016stability}. 

\begin{remark} \textit{(Value Functions as CLFs)} \label{rmk:value}
By inspecting the cost \eqref{eq:cost1} we see that $ V_\gamma^\pi$ is positive definite (since $Q$ is positive definite). Thus, if the right-hand side of \eqref{eq:decay1} is negative for each $x \in \X \setminus \{0\}$, this inequality shows that $V_\gamma^\pi$ is a CLF for \eqref{eq:dynamics}, and that $\pol$ is an asymptotically stabilizing control policy. In other words, $V_\gamma^\pi$ is a CLF which is implicitly learned during the training process. Indeed, many RL algorithms directly learn an estimate of the value function, a fact which we later exploit to learn a CLF for the cartpole swing up-task in Section \ref{sec:examples} using the nominal simulation environment.
\end{remark} 

 Note that the right hand side of \eqref{eq:decay1} will only be negative if $V_\gamma^*(x) + \epsilon_\gamma^\pi(x)<\frac{1}{1-\gamma}Q(x)$. Since from \eqref{eq:cost1} we know that $V_\gamma^*(x)>Q(x)$ for each $x \in \X$, even the optimal policy (which has no optimality gap) will only be stabilizing if $\gamma$ is large enough. On the other hand, for a fixed $\gamma \in \left(0,1\right]$, this inequality also quantifies how sub-optimal a policy can be while maintaining stability. To make these observations more quantitative we make the following assumption:
\begin{assumption}\label{asm:value_bound1}
For each $\gamma \in \sp{0,1}$ there exists $C_\gamma \geq 1$ such that $V_\gamma^*(x) \leq C_\gamma Q(x)$ for each $x \in \X$. 
\end{assumption}
 
Growth conditions of this form are standard in the literature on the stability of approximate dynamic programming \cite{lincoln2006relaxing,postoyan2016stability,gaitsgory2015stabilization,gaitsgory2016stabilization}. Note that, because the running cost $\ell$ is non-negative, we have $C_{\gamma'} \leq C_{\gamma''}$ if $\gamma' \leq \gamma ''$. In particular, the constant $C_1$ upper-bounds the ratio between the one-step cost and the optimal undiscounted value function. When $C_1$ is smaller, the optimal undiscounted policy is more `contractive' and approximate dynamic programming methods converge more rapidly to an optimal solution \cite{lincoln2006relaxing}. Thus, intuitively the constants $C_\gamma \geq 1$ will be smaller when the system is easier to stabilize. The following result is essentially a specialization of the main result from \cite{gaitsgory2016stabilization}:
 \begin{prop}\label{prop1} Let Assumption \ref{asm:value_bound1} hold and let $\gamma \in \sp{0,1}$ and $\pi \in \hat{\Pi}$ be fixed. Further assume that there exists $\delta >0$ such that for each $x \in \X$ we have $i)$ $\epsilon_{\gamma}^\pi(x) \leq \delta Q(x)$ and $ii)$ $C_\gamma+\delta<\frac{1}{1-\gamma}$. Then, $\pi$ asymptotically stabilizes \eqref{eq:dynamics}. 
\end{prop}
\begin{proof}
Combining conditions $i)$ and $ii)$ with equation \eqref{eq:decay1} yields:
\begin{equation*}\label{eq:blah}
    V_\gamma^\pi(F(x,\pi(x))) - V_\gamma^\pi(x) \leq \frac{2}{\gamma}\big(-1 + (1-\gamma)[C_\gamma + \delta]\big)Q(x). \qedhere
 \end{equation*}
 Thus the RHS of the preceding equation will be negative-definite if $C_\gamma + \delta<\frac{1}{1-\gamma}$, which demonstrates the desired result. 
 \end{proof}
\vspace{-.3em}
\begin{remark}\label{rmk:stability} (Stability Properties of the Cost Function) In the following section we will derive an analogous result to Proposition \ref{prop1} for the novel reshaped cost function we propose below. When comparing these results we will primarily focus on the effect of the constants $C_\gamma\geq1$ (and the equivalent constants for the new setting). The  $C_\gamma$ constants can be used to bound how large of a discount factor is need to stabilize the system. In particular, Proposition \ref{prop1} implies that the optimal policy will stabilize the system for each $\gamma$ which satisfies $\gamma >1- \frac{1}{C_\gamma}$. The $C_\gamma$ constants  also characterizes how `robust' the cost function is to suboptimal policies. In particular, for a fixed discount factor, the policy will stabilize the system if $\delta < \frac{1}{1-\gamma} - C_\gamma$. Thus smaller values of the $C_\gamma$ constants permit more suboptimal policies. 
\end{remark}
\vspace{-.5em}
\section{Lyapunov Design for Infinite Horizon Reinforcement Learning}\label{sec:formulation}

Our method uses a positive definite candidate Control Lyapunov Function $W \colon \R^n \to \R$ for the nonlinear dynamics \eqref{eq:dynamics}, and reshapes \eqref{eq:cost1} to our proposed new long horizon cost $\tilde{V}_\gamma^\pi \colon \X \to \R \cup \{\infty\}$:\vspace{-0.1cm}
\begin{align}\label{eq:cost2}
 \tilde{V}_\gamma^\pol(x_0) =& \sum_{k=0}^{\infty}\gamma^{k}\bigg( [W\big(F(x_k,\pol(x_k))\big) - W(x_k)]+ \ell(x_k,\pol(x_k)) \bigg)\\ \nonumber
    & \text{s.t. }  x_{k+1} = F(x_k,\pol(x_k)).
\end{align}
As we shall see below, our method works best when $W$ is in fact a CLF for the system, but still provides benefits when it is only an `approximate' CLF for the system (in a sense we will make precise later). For each $\gamma \in \sp{0,1}$ the new optimal value function is given by:
\begin{equation}
    \tilde{V}_\gamma^{*}(x) = \inf_{\pi \in \Pi}  \tilde{V}_\gamma^{\pi}(x).
\end{equation}

The new cost \eqref{eq:cost2} includes the amount that $W$ changes at each time step, and thus encourages choices of inputs which decrease $W$ over time. In this case, the Bellman equation \cite{bertsekas1996neuro} dictates: 
\begin{equation}\label{eq:best_bellman}
    \tilde{V}_
    \gamma^*(x) = \inf_{u \in \U}\big[ \gamma \tilde{V}_\gamma^*(F(x,u)) +\Delta W(x,u) + \ell(x,u) \big], \ \ \forall x \in \X,
\end{equation}
where $\Delta W(x,u):= W(F(x,u)) -W(x)$. To gain some intuition for the approach let us consider the two extremes where $\gamma = 0$ and $\gamma = 1$. In the case where $\gamma =1$, by inspection we see that $\tilde{V}_1^* = V_1^* - W$ solves the Bellman equation. Plugging in this solution demonstrates that any optimal policy $\tilde{\pi}_1^*$ must satisfy $\tilde{\pi}_1^*(x) \in \arg \min_{u \in \U}[V_1^*(F(x,u)) + \ell(x,u)]$. This is precisely the optimality condition for the original cost \eqref{eq:cost1} when $\gamma =1$, and thus the set of optimal policies for the two problems coincide. Thus, in this case, by embedding the CLF in the cost we are effectively using $W$ as a warm-start initial guess for the optimal value function. In the other extreme where $\gamma = 0$, from \eqref{eq:best_bellman} we see that an optimal policy must satisfy $\tilde{\pi}_{0}^*(x) \in \arg \min_{u \in \U}\big[\Delta W(x,u) + \ell(x,u)\big]$.
Thus, when $\gamma =0$ the optimal policy attempts to greedily decrease the value of the candidate CLF and the one-step cost on the input. As we shall see below, when intermediate discount factors are used, optimal policies may instead decrease the value of $W$ over the course of several steps.


Using the new cost function \eqref{eq:cost2}, each policy must satisfy the new difference equation:
\begin{equation}
    \tilde{V}_\gamma^\pi(x) = \gamma \tilde{V}_\gamma^\pi\big(F(x,\pi(x))\big) + W\big(F(x,\pi(x))\big) - W(x) + \ell(x,\pi(x)).
\end{equation}
In our stability analysis, we will use the following composite function as a candidate CLF for \eqref{eq:dynamics}: 
\begin{equation}\label{eq:new_CLF}
    \tilde{\bm{\mathcal{V}}}_\gamma^\pi(x) = W(x) + \gamma \tilde{V}_\gamma^\pi(x).
\end{equation}
We provide an interpretation of this curious candidate CLF in Remark \ref{rmk:augment} below, but first perform an initial analysis similar to the one presented in the previous section. Defining for each $\pol \in \hat{\Pol}$, $\gamma \in \sp{0,1}$ and $x \in \X$ the new optimality gap:
\begin{equation}
    \tilde{\epsilon}_\gamma^\pi(x) = \tilde{V}_\gamma^*(x) -\tilde{V}_\gamma^\pi(x),
\end{equation}
and following steps analogous to those taken in \eqref{eq:algebra1}-\eqref{eq:decay1}, we can obtain the following: 
\begin{align}
   \tilde{\bm{\mathcal{V}}}_\gamma^\pi\big(F(x,\pi(x))\big) -\tilde{\bm{\mathcal{V}}}_\gamma^\pi(x) &= -\ell(x,\pi(x)) + (1-\gamma)\tilde{V}_\gamma^{\pi}(x)\\
   &= -\ell(x,\pi(x)) + (1-\gamma)\big[\tilde{V}_\gamma^*(x)+ \tilde{\epsilon}_\gamma^\pi(x) \big]\\
   &\leq -Q(x) + (1-\gamma)\big[\tilde{V}_\gamma^*(x)+\tilde{\epsilon}_\gamma^\pi(x)].\label{eq:decay2}
\end{align}
Similar to the analysis in the previous section, we will aim to understand when the right-hand side of \eqref{eq:decay2} is negative, as this will characterize when $\pol$ stabilizes the system. One key difference between the inequalities \eqref{eq:decay1} and \eqref{eq:decay2} is that, while the original value function $V_\gamma^*$ is necessarily positive definite,  $\tilde{V}_\gamma^*$ can actually take on negative values since the addition of the CLF term allows the new running cost in \eqref{eq:cost2} to be negative. As we shall see, this forms the basis for the stability and robustness properties our cost formulation enjoys when $W$ is designed properly.

\begin{remark}\label{rmk:augment} \textit{(Learning Corrections to W)} When the right hand side of \eqref{eq:decay2} is negative for each $x \in \X \setminus \{0\}$, inequality \eqref{eq:decay2} demonstrates that $\tilde{\bm{\mathcal{V}}}_\gamma^\pi$ is in fact a CLF for \eqref{eq:dynamics} and that $\pol$ stabilizes the system (see Theorem \ref{thm:stability}). We can think of $W$ as an `initial guess' for a CLF for the system, while $\gamma \tilde{V}_\gamma^\pol$ is a `correction' to $W$ that is implicitly made by a learned policy $\pol$. Roughly speaking, the larger the discount factor, the larger this correction. Thus, the user can trade-off how much the learned policy is able to correct the candidate CLF $W$ against the additional complexity of solving a problem with a higher discount factor, depending on how `good' they believe the CLF candidate to be. 
\end{remark}

We first state a general stability result for suboptimal policies associated to the new cost, and then discuss how the choice of $W$ affects the stability of suboptimal control policies: 
\begin{assumption}\label{asm:value_bound2}
For each $\gamma \in \sp{0,1}$ there exists $\tilde{C}_\gamma\! \in\! \R$ such that $\tilde{V}_\gamma^*(x) \leq \tilde{C}_\gamma Q(x)$ for each $x \in \X$. 
\end{assumption}
Because the reshaped one-step cost $W(F(x,u)) -W(x) + \ell(x,u)$ can take on negative values, so can the $\tilde{C}_\gamma$ constants. Moreover, in this case it is possibe to have $\tilde{C}_{\gamma '} \geq \tilde{C}_{\gamma''}$ when $\gamma ' \leq \gamma ''$. This is because when larger discount factors are used, the optimal policy can benefit from decreasing $W$ further into the future. The following stability result is analogous to Proposition \ref{prop1}:
\begin{theorem}\label{thm:stability}
Let Assumption \ref{asm:value_bound2} hold and let $\gamma \in \sp{0,1}$ and $\pi \in \hat{\Pi}$ be fixed. Further assume that there exists $\tilde{\delta} >0$ such that for each $x \in \X$ we have $i)$ $\tilde{\epsilon}_{\gamma}^\pi(x) \leq \delta Q(x)$ and $ii)$ $\tilde{C}_\gamma+\tilde{\delta}<\frac{1}{1-\gamma}$. Then, $\pi$ asymptotically stabilizes \eqref{eq:dynamics}. 
\end{theorem}
The proof is conceptually similar to the proof of Proposition 1; we delegate the proof to Appendix \ref{sec:proofs} for brevity. Indeed, note that the conditions for stability under the new cost are essentially identical to those for the previous cost in Proposition \ref{prop1}. 

As alluded to in Remark \ref{rmk:stability}, we will primarily focus on comparing how large the constants $C_\gamma \geq 1$ and $\tilde{C}_\gamma \in \R$ are for the two problems, as they control the discount factor required to learn a stabilizing policy and also the `robustness' of the cost to suboptimal controllers. We provide two characterizations which ensure that $\tilde{C}_\gamma<C_
\gamma$. The first condition is taken from the model-predictive control literature \cite{jadbabaie1999receding,grimm2005model}, where CLFs are used as terminal costs for finite-horizon prediction problems. Proof of the following result can be found in Appendix \ref{sec:proofs}:
\begin{lemma}\label{lemma:decrease}
Suppose that for each $x \in \X$ the following condition holds:
\begin{equation}
\inf_{u \in U} W(F(x,u)) -W(x) + \ell(x,u)\leq 0.
\end{equation}
Then Assumption \ref{asm:value_bound2} is satisfied with constant $\tilde{C}_\gamma  \leq 0$.
\end{lemma}
The hypothesis of Lemma \ref{lemma:decrease} implies that $i)$ $W$ is a true CLF for the system and $ii)$ $W$ dominates the running cost $\ell$, in the sense that $W$ can be decreased more rapidly than $\ell$ accumulates. Effectively, this condition implies that it is advantageous for polices to myopically decrease $W$ at each time step. Consequently, when this condition holds optimal polcies associated to the reshaped costs \eqref{eq:cost2} will stabilize the system for any choice of discount factor. 

The following definition generalizes this condition to cases where $W$ may not be a true CLF for the system but can be decreased over several time-steps: 

\begin{definition}\label{def1}
We say that the candidate CLF $W$ $\bar{\gamma}$-dominates the running cost $\ell$ if for each discount factor $ \bar{\gamma} \leq \gamma \leq 1$ and $x \in \X$ we have $\tilde{V}_\gamma^{*}(x) \leq V_\gamma^{*}(x)$.
\end{definition}

The condition in \eqref{def1} effectively provides a way of characterizing how `close' $W$ is to being a true CLF for the real-world system. In particular, the larger $\bar{\gamma}$ the further into the future RL algorithms must look to see the benefits of decreasing $W$. Our previous discussion, which showed that $\tilde{V}_1^* = V_1^* -W$, demonstrates that every candidate CLF $1$-dominates the cost. Moreover, clearly $W$ can only $0$-dominate the original cost if it is a CLF for the system.
While this condition is more difficult to verify for intermediate values of $\bar{\gamma}$, it provides qualitative insight into how even approximate CLFs for the system can still make it easier to obtain stabilizing controllers.

\begin{remark} \label{remark:robustness}(Robustness of reshaped cost)  When the condition of Lemma \ref{lemma:decrease} is satisfied we will have $\tilde{C}_\gamma \leq 0 < C_\gamma$, implying the new cost enjoys the desirable robustness properties discussed above. When $W$ satisfies the `approximate CLF' condition in Definition  \eqref{def1}, it will only enjoy these benefits when the discount factor is large enough. We leave it as a matter for future work to provide quantitative estimates for the $\tilde{C}_\gamma$ constants in these regimes, and to provide sufficient conditions which ensure $W$ $\bar{\gamma}$-dominates the running cost.
\end{remark}
\vspace{-0.5em}
\section{Examples and Practical Implementations}\label{sec:examples}
We summarize the main results for each of our examples, but leave most details and plots to Appendix \ref{sec:experiment_details}. In every experiment we report, the soft actor-critic algorithm (SAC) \cite{levine_sac} is used as the learning algorithm to optimize the various reward structures we investigate. 

\noindent

\noindent
{\bf Velocity Tracking for A1 Quadruped:} We apply our approach to train a neural network controller which augments and improves a nominal model-based controller \cite{da2021learning} for a quadruped robot using real-world data.  As illustrated by the pink curve in Fig. \ref{fig:a1_experiment} (left), the nominal controller fails to accurately track desired velocities specified by the user. We design a CLF around the desired gait using a linearized reduced-order model for the system. We then collect rollouts of $10 s$ on the robot hardware with randomly chosen desired velocity profiles, and solve an  RL problem using our cost and a discount factor $\gamma=0$. Our approach is able to learn a policy which significantly improves the tracking performance of the nominal controller within 5 minutes (30 episodes) of hardware data, as shown in Fig. \ref{fig:a1_experiment} (left). A video of these results can be found in \url{https://youtu.be/l7kBfitE5n8}, and more details are provided in Appendix \ref{sec:experiment_details}. Furthermore, in Fig. \ref{fig:a1_experiment} (right) we benchmark our approach in simulation against an RL agent trained with a `standard' cost which penalizes the squared error with respect to the desired velocity. As this figure demonstrates, our method is able to rapidly decrease the average tracking error in only around $2$ thousand steps from the environment. In contrast, the benchmark approach is only  able to reach this level of performance for the first time after around $24$ thousand steps. 

\begin{figure}
     \centering
     \begin{subfigure}{0.48\textwidth}
         \centering
         \includegraphics[width=\textwidth]{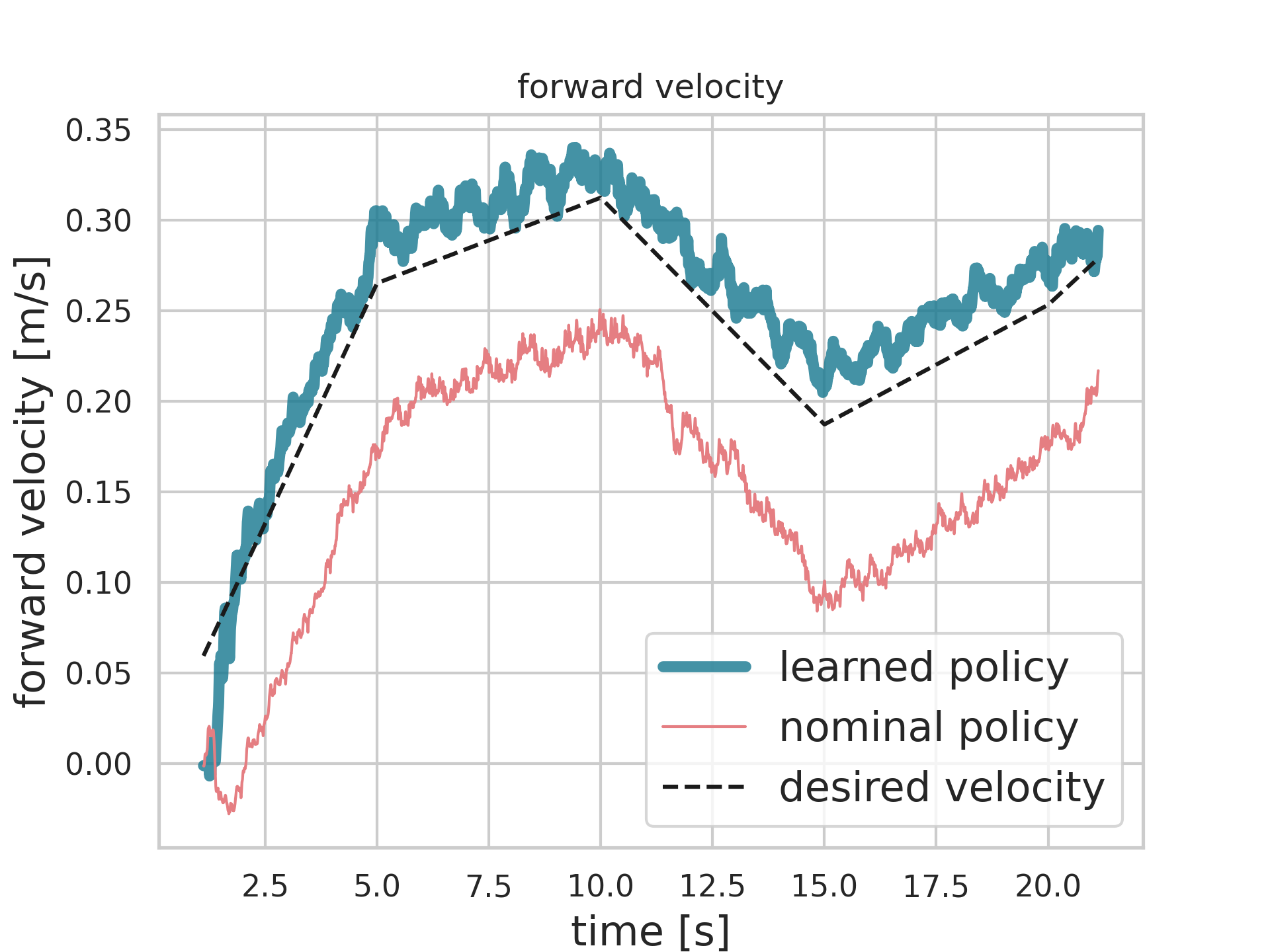}
     \end{subfigure}
     \hfill
     \begin{subfigure}{0.48\textwidth}
         \centering
    \includegraphics[width=\textwidth]{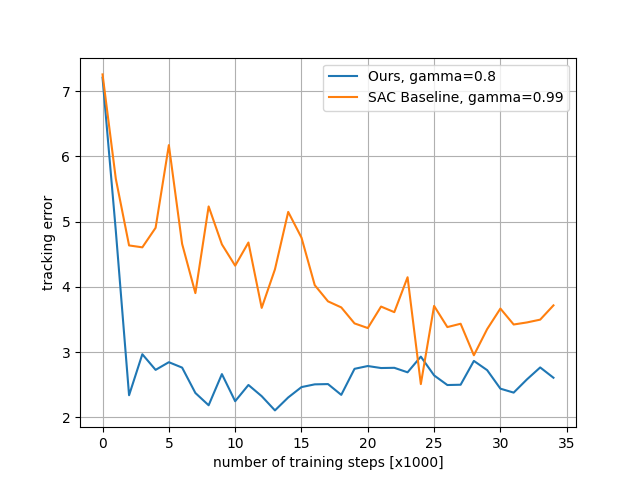}
     \end{subfigure}
        \caption{\small (Left) Plot illustrating improved velocity tracking of the learned policy (in dark green) compared to the nominal locomotion controller (in pink) to track a desired velocity profile (in dashed black line) using our proposed method on the \texttt{Unitree A1} robot hardware. (Right) Plot from the simulated benchmark study illustrating cumulative velocity tracking error (lower is better) over 10s rollouts at different stages of the training. In orange, we show the results of fine-tuning using SAC with a standard RL cost. In blue, we fine-tune using SAC with our reward reshaping method, with a candidate CLF designed on a nominal linearized model of the robot. In both cases, we plot the results using the discount factor that achieved the best performance. }
        \label{fig:a1_experiment}
        \vspace{-1.5em}
\end{figure}

\noindent
\textbf{A1 Quadruped Walking with an Unknown Load:} We attach an un-modeled load to the A1 quadruped, that is equivalent to one-third the mass of the robot. Fine-tuning on hardware the same base controller from the previous set-up where the CLF is designed to stabilize to the target gait, our approach is able to significantly decrease the tracking error to about one-third its nominal value with only one minute of data collected on the robot hardware as illustrated in Fig. \ref{fig:a1-payload-results} in Appendix \ref{sec:experiment_details}. Additionally, in Appendix \ref{sec:experiment_details}, we run a simulated benchmark comparison and verify that our method clearly out-performs the `standard' cost baseline for this task.


\noindent
\textbf{Fine-tuning a Learned Policy for Cartpole Swing-Up:} We fine-tune a swing-up controller for the \texttt{Quanser} cartpole system \cite{quanser_products_2021} using real-world data and an initial policy which was pre-trained in simulation but that does not translate well to the real system. 
Due to the underactuated nature of the system, synthesizing a CLF by hand is challenging. Thus, as alluded to previously, we use a `typical' cost function of the form \eqref{eq:cost1} and a discount factor of $\gamma =0.999$ to learn a stabilizing neural network policy $\pi_\phi$ for a simulation model of the system. Given the discussion in Remark \ref{rmk:value}, we use the value function $V_\theta$ associated with the simulation-based policy as the candidate CLF ($W = V_\theta$) for our reward reshaping formulation \eqref{eq:cost2}. When improving the simulation-based policy $\pi_\phi$ with real-world data, we keep the parameters of this network fixed and learn an additional smaller policy $\pi_{\psi}$ (so that the overall control action is produced by $\pi_\phi + \pi_\psi$) using our proposed CLF-based cost formulation. We solve the reshaped problem with a discount factor $\gamma = 0$ and collect rollouts of $10 s$ on hardware. Our CLF-based fine-tuning approach is able to successfully complete the swing-up task after collecting data from just one rollout. After collecting data from an additional rollout, the controller is reliable and robust enough to recover from several pushes. A video of these experiments can be found in \url{https://youtu.be/l7kBfitE5n8}, and more details and plots of the results are provided in Appendix \ref{sec:experiment_details}. Furthermore, in Appendix \ref{sec:experiment_details} we provide a simulation study comparing a standard fine-tuning approach to our method, showing that our approach is able to more rapidly learn a reliable swing-up policy than the baseline and also achieves a higher reward.

\noindent
\textbf{Fine-tuning a Bipedal Walking Controller in Simulation:} We also apply our design methodology to fine-tune a model-based walking controller \cite{ames2014rapidly} for a bipedal robot with large amounts of dynamics uncertainty. Model uncertainty is introduced by doubling the mass of each link of the robot. The nominal controller fails to stabilize the gait and falls within a few steps. To apply our method, we design a CLF around the target gait as in \cite{ames2014rapidly} to be used in our reward formulation. As a benchmark comparison, we also train policies with a reward which penalizes the distance to the target motion (no CLF term), as is most commonly done in RL approaches for bipedal locomotion which use target gaits in the reward \cite{li2021reinforcement}. Our approach is able to significantly reduce the average tracking error per episode after only 40000 steps of the environment (corresponding to 40 seconds of data), while the baseline does not reach a similar level of performance even after 1.2 million steps, as illustrated in Fig. \ref{fig:rabbit-results} of Appendix \ref{sec:experiment_details}.

\noindent
\textbf{Inverted Pendulum with Input Constraints:} Our final example demonstrates the utility of our method even when $W$ is a crude guess for a CLF for the system, through the use of moderate discount factors. We illustrate this for a simple inverted pendulum simulator by varying the magnitude of the input constraints for the system. We use the procedure from \cite{ames2014rapidly} to design a candidate CLF for the system. Like many CLF design techniques, this approach assumes there are no input constraints and encourages the pendulum to swing directly up. As the input constraints are tightened, $W$ becomes a poorer candidate CLF, as there is not enough actuation authority to decrease $W$ at each time step. Even in this case, in line with the discussion of Remark \ref{remark:robustness}, if a proper discount factor is used, the addition of the candidate CLF in the reward enables our method to rapidly learn a stabilizing controller for each setting of the input bound. These results are presented in Appendix \ref{sec:experiment_details}.

\vspace{-0.3em}
\section{Discussion and Limitations}\label{sec:limitations}
As we have mentioned previously, our approach has several limitations. The cost-shaping technique we introduce in Section \ref{sec:formulation} only provides benefits when $W$ is in-fact a reasonable guess for a CLF for the true system. This requires that the user has a dynamics model which captures the primary features of the environments which affect the structure of CLFs for the system. While the cart-pole simulations we provide in the Appendix \ref{sec:experiment_details} provide some intuition for when this will be the case, further research is needed to better understand in what scenarios we can see significant benefits from our method. Nonetheless, our two hardware experiments provide encouraging initial results which indicate that our method can rapidly learn stabilizing controllers using CLFs which are constructed using a nominal dynamics model. More broadly, there are many exciting avenues for further incorporating Lyapunov design techniques with RL, especially offline learning \cite{levine2020offline}.

\newpage
\bibliography{references.bib}

\newpage
\appendix
\section{Additional Literature Review}\label{sec:extra_lit}

\textbf{Model Predictive Control:}
We briefly review stability results from the model predictive control (MPC) literature, focusing our discussion on the benefits of using a CLF as the terminal cost. In their simplest form, MPC control schemes minimize a cost functional of the form
\begin{align*}
\inf_{\hat{\textbf{u}} \in \U^N} J_{\textit{MPC}}^N(x_k, \hat{\textbf{u}}) &= \sum_{k=0}^{N-1} \big(Q(\hat{x}_k) +R(\hat{u}_k)\big) + \hat{W}(\hat{x}_N)\\
& \text{s.t. } \hat{x}_{k+1} = F(\hat{x}_k, \hat{u}_k), \ \ \hat{x}_{0} = x_k,
\end{align*}
where $x_k$ is the the current state of the real-world system, $N \in \N$ is the prediction horizon, $\{\hat{x}_k\}_{k=0}^{N}$ and $\hat{\bm{u}} = \{\hat{u}_k\}_{k=0}^{N-1} \in \U^N$ are a predictive state trajectory and control sequence, $Q$ and $R$ are as above, and $\hat{W} \colon \R^n \to \R_{\geq 0}$ is the terminal cost which is assumed to be a proper function. The MPC controller then applies the first step of the resulting open loop control and the process repeats, implicitly defining a control law $u_{\textit{MPC}}(x)$. The MPC cost $J^N_{\textit{MPC}}(x_k,\cdot)$ can be thought of as a finite-horizon approximation of the original cost \eqref{eq:cost1} (except that it is defined over an open-loop sequence of control inputs instead of being a cost over policies). 

Stability results from the MPC literature focus primarily on the effects of the prediction horizon $N$ and the choice of terminal cost $\hat{W}$. Under mild conditions, for any choice of terminal cost (including $\hat{W}(\cdot) \equiv 0$), the user can guarantee that the MPC scheme stabilizes the system on any desired operating region by making the prediction horizon $N$ sufficiently large \cite{jadbabaie2005stability,grimm2005model}. Thus, there is a clear connection between the explicit prediction horizon $N$ in MPC schemes and the discount factor $\gamma$, as both need to be sufficiently large if a stabilizing controller is to be obtained (since trajectory optimization problems with longer time horizons are generally more difficult to solve). Indeed, in \cite{postoyan2016stability} it was pointed out that the \emph{implicit prediction horizon} $\frac{1}{1-\gamma}$, a factor which shows up in the stability conditions in Proposition \ref{prop1}, plays essentially the same role in stability analysis as $N$ for an MPC scheme with no terminal cost when the running cost is $\ell = Q +R$. Thus, much like the `typical' policy optimization problems discussed in Section \ref{sec:og_cost}, MPC schemes with no terminal cost (or one which is chosen poorly) may require an excessively long prediction horizon to stabilize the system. 

Fortunately, the MPC literature has a well-established technique for reducing the prediction horizon needed to stabilize the system: use an (approximate) CLF for the terminal cost $\hat{W}$ \cite{jadbabaie2001unconstrained,jadbabaie2005stability,grimm2005model}. Indeed, roughly speaking, these results guarantee that for \emph{any prediction horizon} $N \in \N$ the MPC scheme will be stabilizing if $\hat{W}$ is a valid CLF for the system. Extensive empirical evidence \cite{jadbabaie1999receding} and formal analysis \cite{jadbabaie2001unconstrained} has demonstrated that well-designed CLF terminal costs reduce the prediction horizon needed to stabilize the system on a desired set and increase the robustness of the overall MPC control scheme \cite{grimm2004examples}. Thus, in many ways our cost-reshaping approach can be seen as a way to obtain these benefits in the context of infinite horizon model-free reinforcement learning.

\section{Asymptotic Stability and Lyapunov Theory}\label{sec:stability}

\subsection{Asymptotic Stability and Lyapunov Theory}

Next, we briefly introduce the elements from stability theory and Lyapunov theory which we use extensively throughout the paper.

\subsection{Notation and Terminology} We say that a function $W \colon \R^n \to \R$ is \emph{positive definite} if $W(0) =0 $ and $W(x)>0$ if $x \neq 0$.
Let $\alpha \colon \left[0,\infty \right) \to [0,\infty)$ be a continuous function. We say that $\alpha$ is in class $\K$ (denoted $\alpha \in \K$) if $\alpha(0) =0$ and $\alpha$ is strictly increasing. If in addition we have $\alpha(r)\to \infty$ as $r \to \infty$ when we say that $\alpha$ is in class $\K_\infty$ (denoted $\alpha \in \K_\infty$). Let $\beta \colon \left[0,\infty\right) \times \left[0,\infty\right)$ be a continuous function. We say that $\beta$ is in class $\KL$ if for each fixed $t \in \left[0,\infty\right)$ the function $\beta(\cdot,t)$ is in class $\K$ and for each fixed $r \in \left[0,\infty \right)$ we have $\beta(r,t) \to 0$ as $t \to \infty$. 

\subsection{Basic Stability Results}

\begin{definition} We say that the closed loop system $x_{k+1} = F(x_k, \pi(x_k))$ is \emph{asymptotically stable on the set $D \subset \R^n$} if there exists $\beta \in \KL$ such that for each initial condition $x_0 \in D$ and $k \in \N$ the closed-loop trajectory satisfies:
\begin{equation}
\|x_k\|_2 \leq  \beta(\|x_0\|_2,k).
\end{equation}
Analogously, if the preceding condition holds then we say that $\pol$ \emph{asymptotically stabilizes} \eqref{eq:dynamics}.
\end{definition}
In words, the definition says that $\pol$ asymptotically stabilizes \eqref{eq:dynamics} if all trajectories of the closed-loop system $x_{k+1} = F(x_k,\pol(x_k))$ converge to the origin. Asymptotic stability is a difficult property to verify directly as it requires reasoning about the infinite-horizon behavior of trajectories. Lyapunov functions are a powerful analysis tool which can verify asymptotic stability with a `one-step' criterion: 

\begin{definition}
 We say that the positive definite function $W \colon \R^n \to \R$ is a Lyapunov function for the closed-loop system $x_{k+1} = F(x_k,\pol(x_k))$ if for each $x \in \R^n$ we have:
 \begin{equation}\label{eq:lyap_def}
     W(F(x,\pol(x))) - W(x) < 0.
 \end{equation}
\end{definition}
Intuitively, the Lyapunov function $W$ can be thought of as an energy-like function for the closed loop system $x_{k+1} = F(x_k,\pol(x_k))$. In this light, the condition \eqref{eq:lyap_def} ensures that the 'energy' of the closed-loop system is decreasing at each point in the state-space. This condition guarantees that the closed-loop system is asymptotically stable \cite{sastry2013nonlinear}, and is a simple algebraic condition. Note that while control Lyapunov functions are defined formally for the open-loop dynamics \eqref{eq:dynamics}, a Lyapunov function is defined for a particular set of closed-loop dynamics. That is, a control Lyapunov function $W$ for $x_{k+1} = F(x_k,u_k)$ becomes a Lyapunov function for the closed-loop dynamics $x_{k+1} = F(x_k,\pi(x_k))$ after we apply a control law $\pi$ which satisfies $W(F(x,\pi(x))) - W(x) <0$ for each $x \in \X$.

\section{Missing Proofs and Intermediate Results}\label{sec:proofs}

\begin{lemma}\label{lemma:lyap_bound}
The composite function $\tilde{\bm{\mathcal{V}}}_\gamma^\pi = W +\gamma  \tilde{V}_\gamma^\pi \colon \X \to \R \cup \{\infty\}$ is positive definite.
\end{lemma}
\begin{proof}
Note that we can re-write the reshaped cost \eqref{eq:cost2} as 
\begin{equation}
    \tilde{V}_\gamma^\pi(x_0) = \sum_{k=0}^\infty \gamma^k \bigg( [W(x_{k+1}) - W(x_k) + \ell(x_k,\pi(x_k))] \bigg),
\end{equation}
where $\{x_k\}_{k=0}^\infty$ is the state trajectory generated by the policy $\pi$ from the initial condition $x_0 \in \X$. By rearranging terms we can rewrite this expression as:
\begin{equation}
    \tilde{V}_\gamma^\pi(x_0) = -W(x_0) + (1-\gamma) \sum_{k=0}^{\infty} \gamma^{k}W(x_{k+1}) + \sum_{k=0}^\infty \gamma^k \ell(x_k,\pi(x_k)) > -W(x_0) +Q(x_0)
\end{equation}
where  we have used the fact that $W$ and $\ell$ are both non-negative, and that $\ell(x_0, \pi(x_0))>Q(x_0)$.
Thus, using this expression we see that 
\begin{equation}
   \tilde{\bm{\mathcal{V}}}_\gamma^\pi(x_0)= W(x_0) +  \gamma \tilde{V}_\gamma^\pi(x_0)>(1-\gamma)W(x_0) + \gamma Q(x_0),
\end{equation}
 Since $Q$ and $W$ are assumed to be positive definite functions this demonstrates that $\bm{\mathcal{V}}_\gamma^\pi$ is in fact positive definite, since a convex combination of positive definite functions is positive definite. The proof is concluded by noting that the choice of $\gamma$ and $\pol$ is arbitrary, and thus the conclusion that $\bm{\mathcal{V}}_\gamma^\pi$ is positive definite holds for all policies and discount factors. 
\end{proof}

\subsection{Proof of Theorem \ref{thm:stability}}
\begin{proof}
Lemma \ref{lemma:lyap_bound} demonstrates that $\tilde{\bm{\mathcal{V}}}_\gamma^\pi = W + \gamma \tilde{V}_\gamma^\pi \colon \X \to \R \cup \{\infty\}$ is a positive definite function. Using the hypotheses of the results with the inequality \eqref{eq:decay2} we obtain
\begin{equation}
    \tilde{\bm{\mathcal{V}}}_\gamma^\pi\big(F(x,\pi(x))\big) -\tilde{\bm{\mathcal{V}}}_\gamma^\pi(x)\leq (-1 + (1-\gamma)[\tilde{C} + \tilde{\delta}])Q(x).
\end{equation}
Note that if $\tilde{C}+ \tilde{\delta} <\frac{1}{1-\gamma}$ then the right hand side of \eqref{eq:blah} will be negative definite, which establishes that $\pol$ asymptotically stabilizes the system.
\end{proof}

\subsection{Proof of Lemma \ref{lemma:decrease}}
\begin{proof}
 Consider a policy $\bar{\pol} \in \Pol$ defined for each $x \in \X$ by:
\begin{equation}
    \bar{\pi}(x) \in \arg \inf_{u \in \U} W(F(x,u)) - W(x) +\ell(x,u) \leq 0,
\end{equation}
where the preceding inequality follows directly from the assumptions made in the Lemma. Next, for a given initial condition $x_0 \in \X$ let $\{x_k\}_{k=0}^\infty$ be the state trajectory generated by $\bar{\pi}$. The corresponding reshaped cost is given by
\begin{align}
    \tilde{V}_\gamma^{\bar{\pi}}(x_0) &= \sum_{k=0}^{\infty}\gamma^{k}\bigg( [W\big(F(x_k,\bar{\pol}(x_k))\big) - W(x_k)]+ \ell(x_k,\bar{\pol}(x_k)) \bigg)\\
    &\leq \sum_{k=0}^{\infty} \gamma^k(0)\\
    &\leq 0,
\end{align}
which demonstrates the desired result, since the initial condition and discount factor were chosen arbitrarily. 
\end{proof}

\section{Additional Experiment Details}\label{sec:experiment_details}
We now provide more details of the experimental results reported in Section \ref{sec:examples} and also additional evaluations. While we have chosen to minimize costs in the main portion of the paper, as this is more consistent with the notation used in the literature on Lyapunov theory and the stability of dynamic programming, most RL algorithms take in rewards that are to be maximized. Thus, for the sake of consistency with practical implementations, in this section we report the reward functions used in our code, which are simply the costs from before with the sign flipped. 

For training from hardware data, we used asynchronous off-policy updates, similar to the framework presented in \cite{gu2017deep}. In particular, we have two separate threads, with one running episodes on the hardware system with the latest available policy and adding the transition data to the replay buffer, and the other one sampling from this buffer and performing the actor and critic updates. We only synchronize the policy network weights at the beginning of each episode. 

\subsection{A1 Quadruped Results}

To illustrate the efficacy of our approach, we run two sets of experiments with the A1 robot: 1) accurately tracking a target velocity when the gains $k_p$ and $k_d$ are not well tuned (Section \ref{sec:examples}); and 2) accurately tracking the height of the robot with an unknown load attached to it. Here we provide additional details of experiments related to these experiments. For both settings, we use the locomotion controller presented in \cite[Section 3.2]{da2021learning} as our nominal baseline controller. This controller uses a linearized rigid-body model to formulate a quadratic-program (QP)-based controller to track a desired body pose of the robot. Specifically, the following QP is solved to obtain the ground reaction forces $f$ for the feet in contact with the ground:
\begin{align}
    \min_f ~ & \|\mathbf{M}f - \tilde{g} - \ddot{q}_d \|_Q + \|f\|_R \\
    \text{s.t.} ~  &f_z \geq 0,\nonumber \\ 
    &-\mu f_z \leq f_x \leq \mu f_z,  \nonumber \\
     &-\mu f_z \leq f_y \leq \mu f_z, \nonumber
\end{align}

\noindent
where $\mathbf{M}$ is the inverse inertia matrix of the rigid body, $\tilde{g} := [0, 0, g, 0, 0, 0]$ denotes the acceleration due to gravity and $\ddot{q}_d \in \mathbb{R}^6$ are the desired pose accelerations of the robot's body. In particular, the desired accelerations are obtained using a PD controller,
\begin{equation}
    \ddot{q}_d = -k_p(q - q_d) - k_d(\dot{q} - \dot{q}_d), \label{eq:desired ddq}
\end{equation}
\noindent
with $q \in \mathbb{R}^6$ denoting the robot's body pose.

Next, we provide further details for each set of experiments on the A1 robot. 
\noindent
\subsubsection{Velocity Tracking for A1 Quadruped}
When the feedback gains $k_p, k_d \in \mathbb{R}^{6}$ are not well tuned, large tracking errors in the forward speed of the robot can persist as illustrated in Fig. \ref{fig:a1_experiment} (left). To compensate for the increased tracking error, we learn a policy $\pi_\theta$ (MLP with two hidden layers of size $32\times 32$) that outputs an additional acceleration term in \eqref{eq:desired ddq}, making the final desired acceleration $\ddot{q}_d =  -k_p(q - q_d) - k_d(\dot{q} - \dot{q}_d) + \pi_{\theta}$. $\pi_{\theta}$ can therefore be viewed as a learned fine-tuning policy with respect to a model-based controller. The observations for the RL agent include the forward and lateral velocity, the roll and pitch orientation and the desired forward velocity of the robot. The actions include offsets to the desired forward and lateral accelerations. 

The policy $\pi_\theta$ is learned directly on the robot hardware using a CLF $W$ designed for the nominal rigid body dynamics of the robot following the procedure described in \cite{ames2014rapidly}. For training, we use SAC \cite{levine_sac} with the reward $r_k = \frac{\left(W(F(x_k, u_k))-W(x_k) \right)}{\Delta t_k} + \lambda \|u_k\|^2$. The CLF term in the reward allows us to use a discount factor $\gamma=0$, which considerably reduces the complexity of the learning problem. Indeed, within only 5 minutes of data collected from the robot hardware, our method is able to significantly reduce the tracking error in the forward velocity compared to the nominal locomotion controller, as shown in Figure \ref{fig:a1_experiment} (left).

\subsubsection{Height Tracking with an Unknown Load}
 
\begin{figure}
    \centering
    \includegraphics[width=\textwidth]{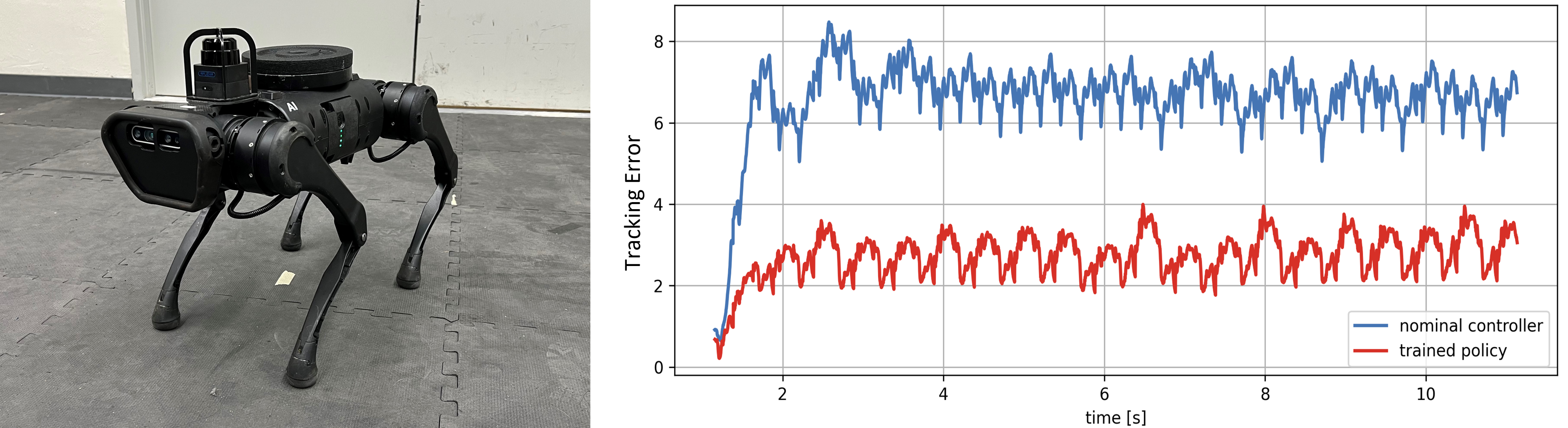}
    \caption{\small Comparison between nominal controller and learned policy after training on 60s of real-world data on the A1 robot with an added 10lb weight. The learned policy is able to significantly reduce the tracking error caused by the added weight.}
    \label{fig:a1-payload-results}
\end{figure}

\begin{figure}
    \centering
    \includegraphics[width=.6\textwidth]{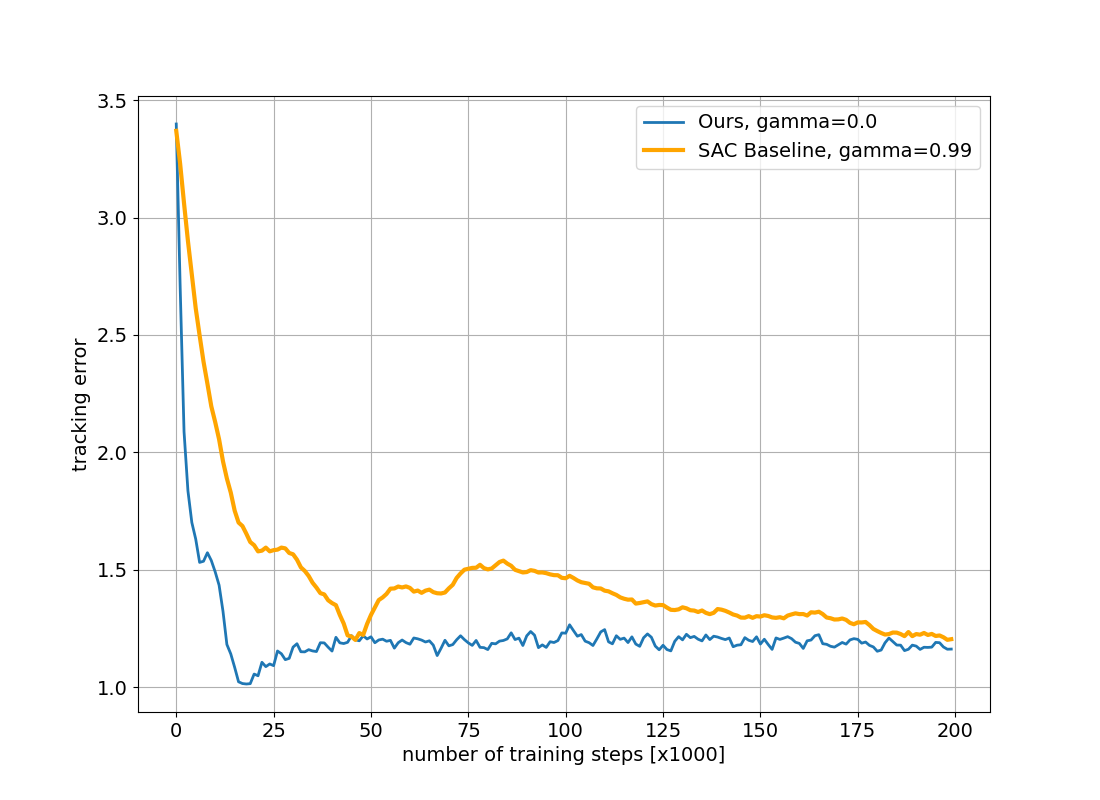}
    \caption{\small Cumulative gait tracking error (lower is better) over 10s rollouts at different stages of the simulated fine-tuning benchmark comparison of the A1 quadruped with an unknown load. In orange, we show the results of fine-tuning using SAC with a standard RL cost which penalizes the distance to the desired gait with a discount factor of $\gamma = 0.99$. In blue, we plot the performance of our cost reshaping method with SAC and a discount factor of $\gamma = 0$. For both cost formulations, we plot the discount factor that led to the best performance. }
    \label{fig:a1_comparison_2}
\end{figure}

In this experiment, we use the same base controller and an equivalent offset policy $\pi_{\theta}$ as in the previous set-up and attempt to track a target gait. The CLF is designed to stabilize to the target gait as in the previous experiment. Figure \ref{fig:a1-payload-results} plots the tracking error of the learned controller versus the nominal controller after only 1 minute of training data. As the figure demonstrates, our approach is able to significantly decrease the error to about one-third its nominal value with only a small amount of data. 

To verify that our method out-performs the baseline for this task, we run a simulated benchmark comparison similar to the A1 simulation study for velocity tracking that was presented in Section \ref{sec:examples} of the paper. For this case, we reproduce the unknown load hardware experiment in simulation by adding a 10lb weight to the robot. When testing our method, we again use SAC with the same reward formulation from the hardware experiments above. For the baseline reward, we penalize the distance to the target that we want to track. Figure \ref{fig:a1_comparison_2} depicts the best results that we have been able to obtain for each cost formulation across different discount factors and training hyper-parameters. As Fig. \ref{fig:a1_comparison_2} depicts, our approach quickly converges to a stable walking controller which closely tracks the references after only around $ 22$ thousand steps of the environment. The baseline does not match this performance until it has had access to around $ 48$ thousand steps, and takes much longer to consistently approach the performance of our method.

\subsection{Cartpole Results}
We first provide plots and give additional details for the cartpole experiments presented in Section \ref{sec:examples}. Then, we present a comparison of the performance of our approach with respect to a typical fine-tuning method on a simulator of the cartpole system.

\subsubsection{Additional Details of the Cartpole Hardware Fine-tuning Experiments}

\begin{figure}
\begin{center}
\includegraphics[width=\textwidth,trim={0 1.8cm 0 0},clip]{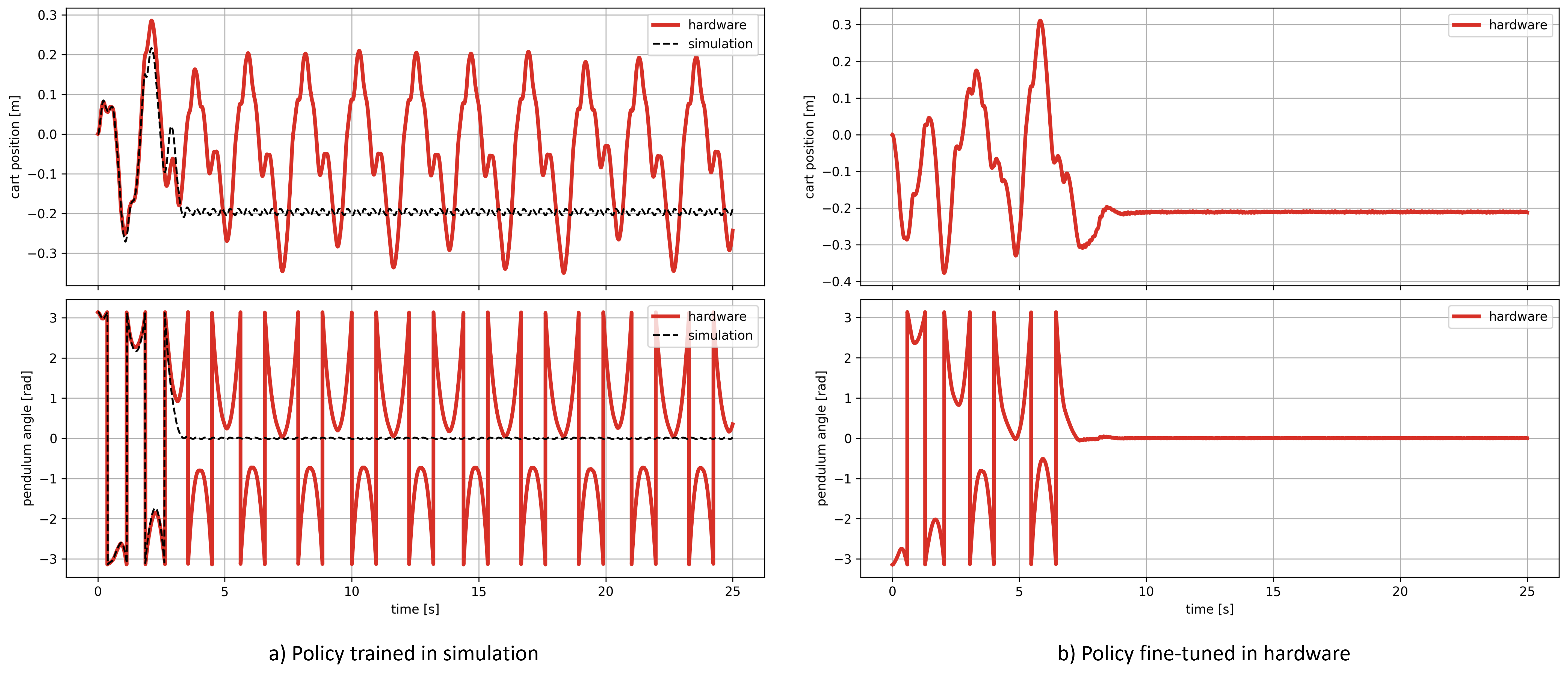}
\end{center}
\caption{\small Experimental plots of the cart position and pendulum angle of the cartpole system. (left) The policy trained only in simulation fails to bring the real cartpole system to the upright position; (right) by fine-tuning the learned policy with $20 s$ of real-world data using our CLF-based reward function, we obtain a successful policy.}
\label{fig:cartpole-results}
\end{figure}

For the cartpole experiments presented in Section \ref{sec:examples}, we used a Quanser Linear Servo Base Unit with Inverted Pendulum \cite{quanser_products_2021}, with a pendulum length of $60$cm. The system has $4$ states, $x = [p,\ \alpha,\ \dot{p},\ \dot{\alpha}] \in \R^4$, corresponding to the cart position $p$, the pendulum angle $\alpha$, and their respective velocities. The control input is the voltage applied to the motor that actuates the cart $u \in \R$.

We first train a SAC agent in simulation using a `conventional' RL reward that penalizes the distance to the equilibrium, control effort, and includes a penalty if the cart goes off-bounds $r(x_k,u_k) = -0.1\ (5\alpha_k^2 + p_k^2 + 0.05 u_k^2) -5\cdot 10^3 \cdot \mathbbm{1}(|p_k|\geq 0.3)$. The observations of the RL agent are state measurements, the actions are direct voltage commands with limits set to $|u|<10$V as specified by the manufacturer, and the simulation is run at $100$Hz. In order to obtain a stabilizing swing-up policy with this traditional reward, a high discount factor is needed, so we use $\gamma = 0.999$. After around $15$ thousand seconds of simulation data with a learning rate of $5\cdot 10^{-4}$, the RL agent learns to consistently swing-up and balance the pendulum at the upright position in simulation. However, when deployed on the cartpole hardware system, the policy from simulation fails to obtain successful swing-up behaviors due to the sim-2-real gap, as shown in the attached video.

To tackle these issues, we exploit the fact that SAC uses a feedforward neural network to approximate the discounted value function of the problem, and we use this approximate value function (after 18,600 seconds of data) as a CLF candidate to fine-tune the learned policy directly on hardware.

Thus, we learn on hardware a fine-tuning policy $u_\psi$ (MLP with 2 hidden layers of $64 \times 64$) whose actions are added to the ones of the policy trained on simulation $u_\phi$ (MLP with 2 hidden layers of $400 \times 300$). The episodes are $10$ seconds long, and the policy is run at $500$Hz, with each episode consisting of $5000$ data points. The action space limits for this new policy are set to $|u_\psi|< 4$V but we still have a saturation of the total voltage $|u_\phi + u_\psi|<10$V. The reward for this new policy is $\hat{r}(x_k,u_k) = \Delta V_\theta(x_k,u_k) -0.1\cdot(5\alpha_k^2 + p_k^2 + 0.05 u_k^2)$, where $V_\theta$ is the value function network of the SAC agent that was trained in simulation. This allows us to set the discount factor $\gamma = 0$ for the offset policy learned on hardware and therefore greatly reduce the complexity of the learning problem. After only one episode of $10$ seconds of real-world data we obtain a policy that manages to swing-up the pendulum to the upright position, and stabilizes it at the top. However, the behavior near the top is not smooth, and it fails for some different initial conditions. After training with another episode of 10 seconds of data, we obtain a policy that consistently manages to swing-up and balance the pendulum at the top, while the cart stays in-bounds. The plots in Fig. \ref{fig:cartpole-results} (right) show the cart position and the pendulum angle when deploying the fine-tuned policy in the real Quanser cartpole system. The plots in Fig. \ref{fig:cartpole-results} (left) show the results when using the policy that has been only trained in simulation, and how its performance is very different when deployed in simulation vs in hardware. A video with the results of the cartpole experiments can be found in \url{https://youtu.be/l7kBfitE5n8}, and a sequence of snapshots of a successful experiment that uses the fine-tuned policy can be found in Figure \ref{fig:cover}.

\subsubsection{Cartpole Simulation Baseline Comparison with a Typical Fine-tuning Method}

\begin{figure}
    \centering
    \includegraphics[width=\textwidth]{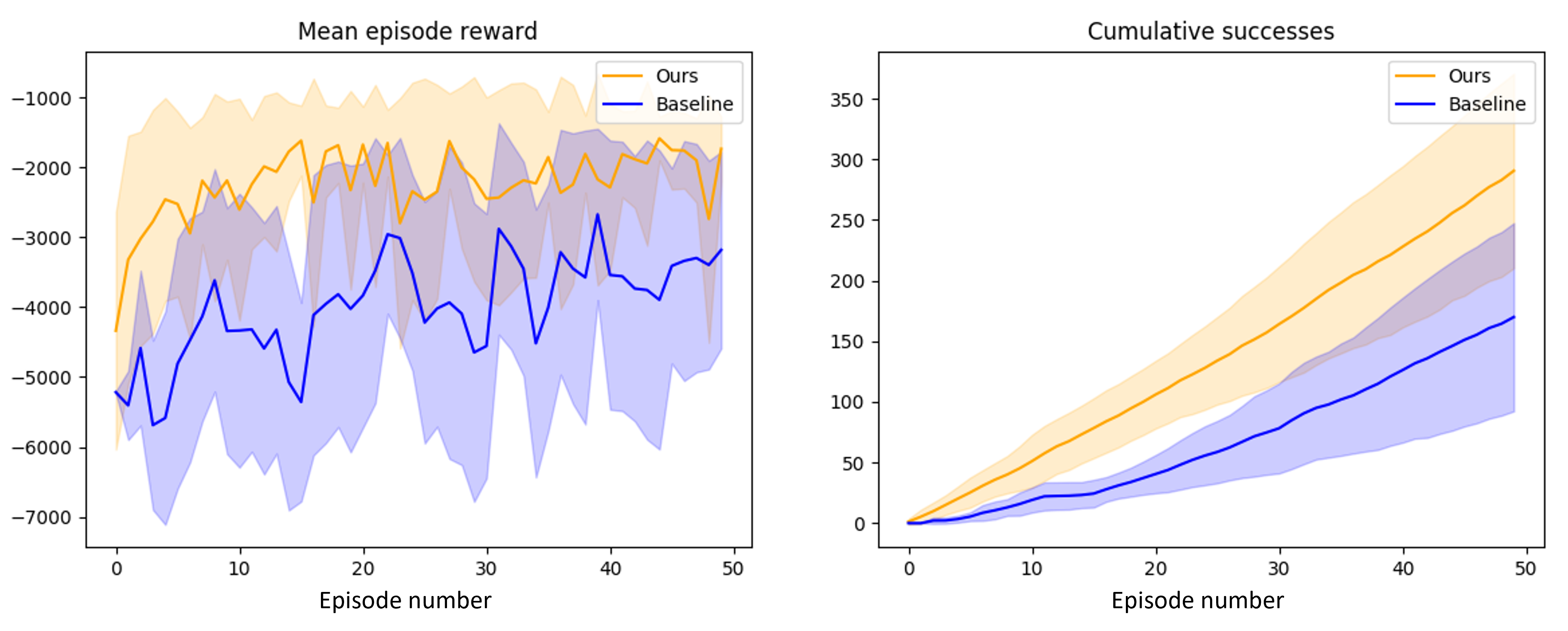}
    \caption{\small Comparison of the simulation results of fine-tuning a cartpole swing-up policy after adding model mismatch. A policy trained on a nominal dynamics model of the cartpole fails when deployed on the new dynamics. In blue, we show the results of continuing to train the agent with the original costs and discount factor. In orange, we fine-tune using our reshaping method with the pre-trained value function and a discount factor of $\gamma=0$. For each episode of training on the new dynamics model, we compare the performance of both methods when running the cartpole from 10 initial conditions: (on the left) the average original reward without the CLF term, and (on the right) the cumulative number of successful swing-ups. The plots show the mean and standard deviation of the results over 10 different training random seeds.}
    \vspace{-1em}
    \label{fig:cartpole comparison}
\end{figure}

As explained at the beginning of the paper, previous work has shown that using hardware data to fine-tune a policy that has been pre-trained in simulation is a powerful approach to tackle the sim-2-real gap problem (e.g. \cite{smith2021legged,julian2020never,julian2020efficient,mandi2022effectiveness}). These methods typically take the RL agent trained in simulation and continue its learning process using hardware data, the original cost function and discount factor (see e.g. \cite{smith2021legged}). In contrast, our proposed approach stops the simulation training of $u_\phi$ and learns a smaller offset policy $u_\psi$ from hardware data using a separate learning process that has a different reward function $\hat{r}$ (with the CLF candidate being the learned value function in simulation) and a smaller discount factor (in this case $\gamma = 0$).


In Figure \ref{fig:cartpole comparison}, we compare in simulation the results of using this standard fine-tuning approach with those obtained with our method. For both approaches, we first pre-train a policy $\pi_\phi$ and value function $V_\theta$ on a nominal set of dynamics using SAC and the reward $r(x_k,u_k) = -0.1\ (5\alpha_k^2 + p_k^2 + 0.05 u_k^2) -5\cdot 10^3 \cdot \mathbbm{1}(|p_k|\geq 0.3)$, and then perturb the parameters of the simulator to introduce model mismatch for the fine-tuning phase. Specifically, we increase the weight and friction of the cart by $200\%$; and the mass, inertia and length of the pendulum by a $25\%$. After doing this, we randomly sample $10$ initial conditions around the downright position ($ -0.05 m \leq p_0 \leq 0.05 m$, $-\pi + 0.05 rad \leq \alpha_0 \leq \pi - 0.05 rad$, $ -0.05 m/s \leq \dot{p}_0 \leq 0.05 m/s$, $ -0.05 rad/s \leq \dot{\alpha}_0 \leq 0.05 rad/s$). We label a trial as success if within 10 seconds of simulation, the pendulum is stabilized in the set $-0.12 rad < \alpha < 0.12 rad$, $-0.3 rad/s < \dot{\alpha} < 0.3 rad/s$ and the cart never gets out of bounds ($|p|< 0.3$). The policy $u_\phi$ trained with data from the nominal dynamics model does not succeed for any of the $10$ initial conditions due to the model mismatch. The baseline in Figure \ref{fig:cartpole comparison} is obtained by emptying the replay buffer and using data from the new environment to continue the training process of $u_\phi$ with the same reward $r(x_k,u_k)$. On the other hand, as with the hardware experiments, our method takes the learned value function $V_\theta$ from the nominal dynamics model and learns an offset policy $u_\psi$ using the modified reward $\hat{r}(x_k,u_k) = \Delta V_\theta(x_k,u_k) -0.1\cdot(5\alpha_k^2 + p_k^2 + 0.05 u_k^2)$. In Figure \ref{fig:cartpole comparison}, we plot for $10$ training random seeds the average original reward $r(x_k,u_k)$ and the cumulative number of successes of the validation episodes ran from the initial conditions mentioned above. The x axis is the number of rollouts of fine-tuning data (each rollout consists of $10$ seconds of data). As this figure clearly demonstrates, our approach is able to more rapidly learn a reliable swing-up controller than the baseline. Moreover, as the plot on the left displays, even though we are no longer optimizing for the original reward, by rapidly converging to a stabilizing controller our method still performs better on the original reward than the benchmark.

The above results show that our approach effectively serves to fine-tune policies when the dynamics of the system change. In fact, we have artificially added a severe model mismatch and shown that we can adapt to the new dynamics with a discount factor of $0$. This is because the original value function is still a `good' CLF candidate for the new system. However, if the change in the dynamics is drastic, or if the overall shape of the motion required to complete the task has to be greatly modified, then the value function from the original dynamics may not be a good CLF candidate, and our method might fail. We have observed that for the cartpole example our method is very robust to variations in the parameters of the cart dynamics (in fact, in the above example we are multiplying both friction and mass of the cart by a factor of $3$), but that if we drastically reduce the length and mass of the pendulum by a $50\%$, our method fails. We hypothesize that this might be related to the underactuated nature of the pendulum dynamics. An interesting direction for future work would therefore be to study under which conditions the original value function retains the CLF properties for a new set of dynamics.

\subsection{Bipedal Walking Results}

\begin{figure}
     \centering
     \begin{subfigure}[b]{0.1\textwidth}
         \centering
         \includegraphics[width=\textwidth,trim=16cm 3.5cm 15cm 8cm, clip]{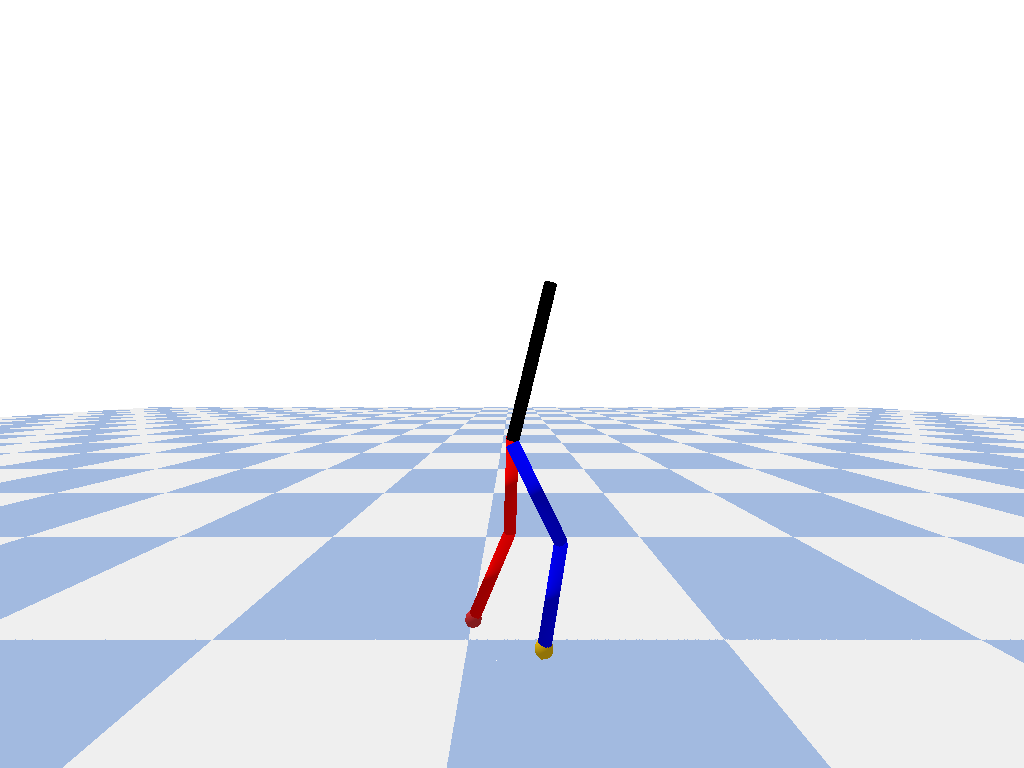}
     \end{subfigure}
     \hspace{-0.3cm}
     \begin{subfigure}[b]{0.1\textwidth}
         \centering
         \includegraphics[width=\textwidth,trim=16cm 3.5cm 15cm 8cm, clip]{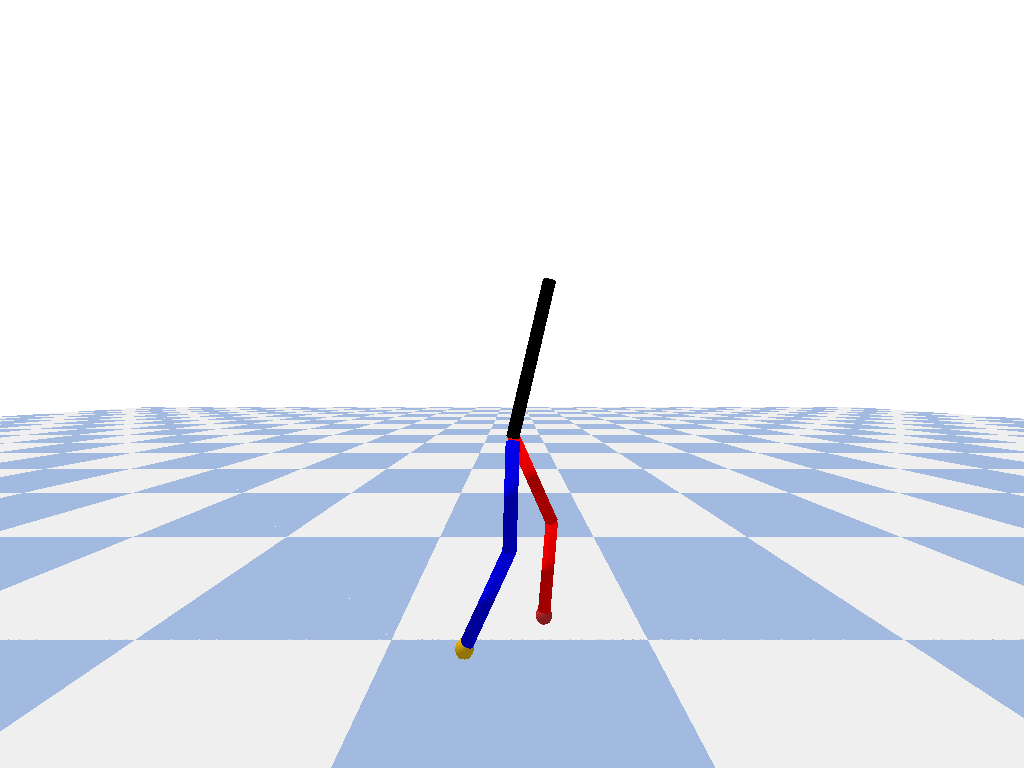}
     \end{subfigure}
     \hspace{-0.3cm}
     \begin{subfigure}[b]{0.1\textwidth}
         \centering
         \includegraphics[width=\textwidth,trim=16cm 3.5cm 15cm 8cm, clip]{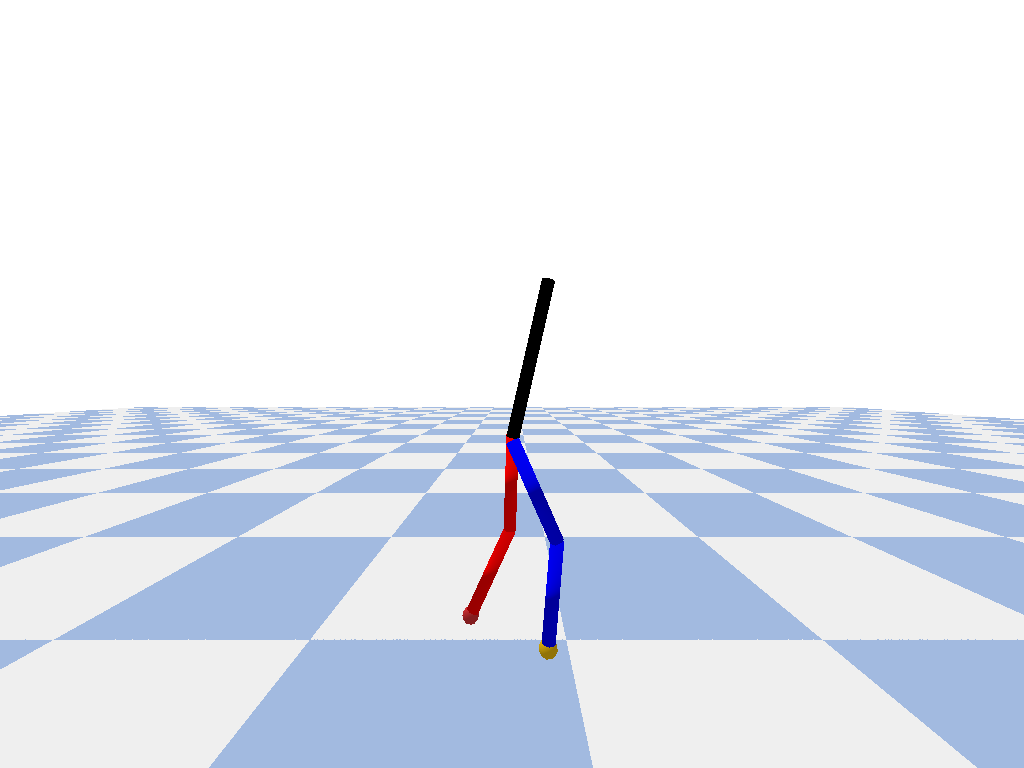}
     \end{subfigure}   
     \hspace{-0.3cm}
     \begin{subfigure}[b]{0.1\textwidth}
         \centering
         \includegraphics[width=\textwidth,trim=16cm 3.5cm 15cm 8cm, clip]{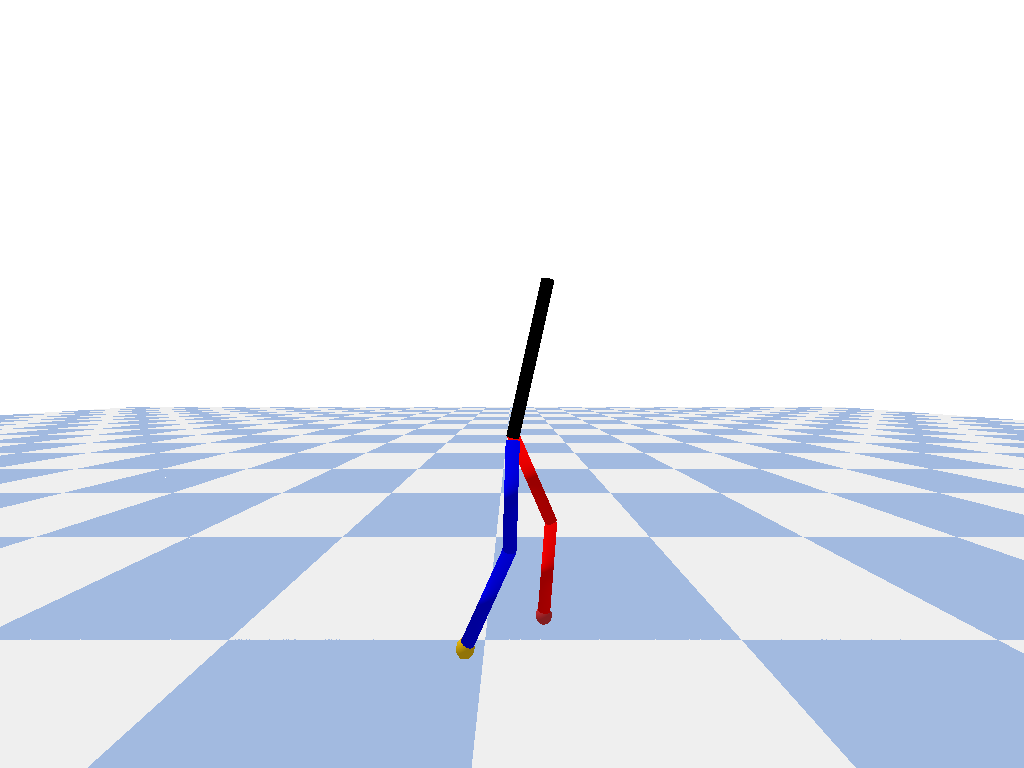}
     \end{subfigure}   
     \hspace{-0.3cm}
     \begin{subfigure}[b]{0.1\textwidth}
         \centering
         \includegraphics[width=\textwidth,trim=16cm 3.5cm 15cm 8cm, clip]{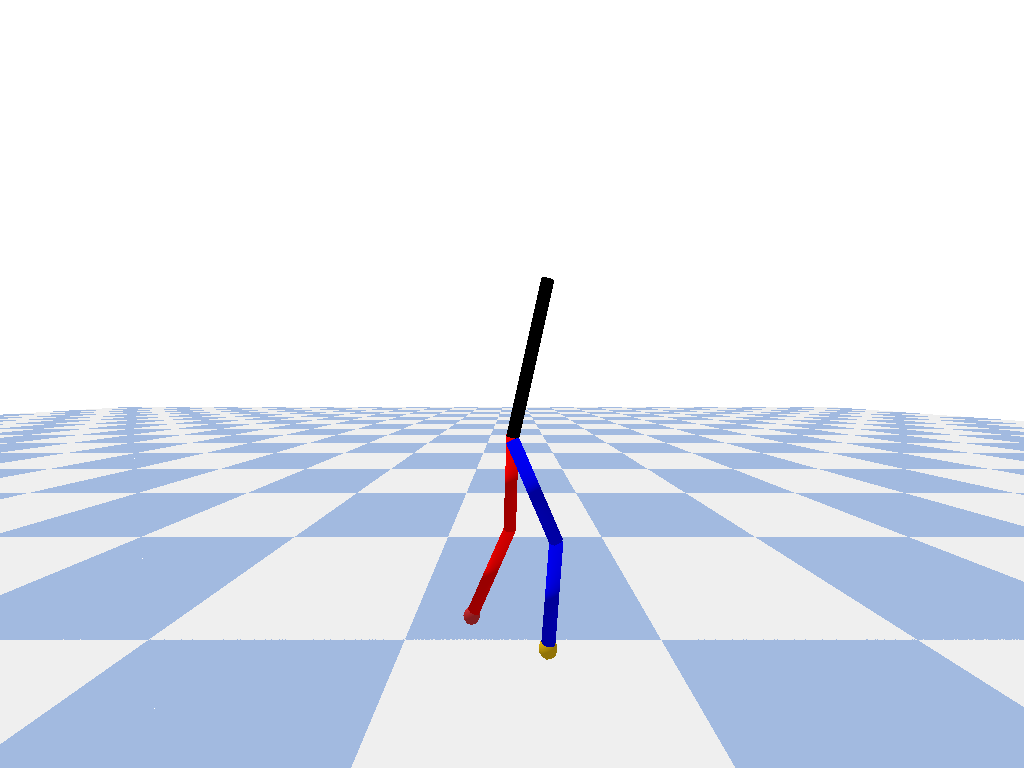}
     \end{subfigure}   
     \hspace{-0.3cm}
     \begin{subfigure}[b]{0.1\textwidth}
         \centering
         \includegraphics[width=\textwidth,trim=16cm 3.5cm 15cm 8cm, clip]{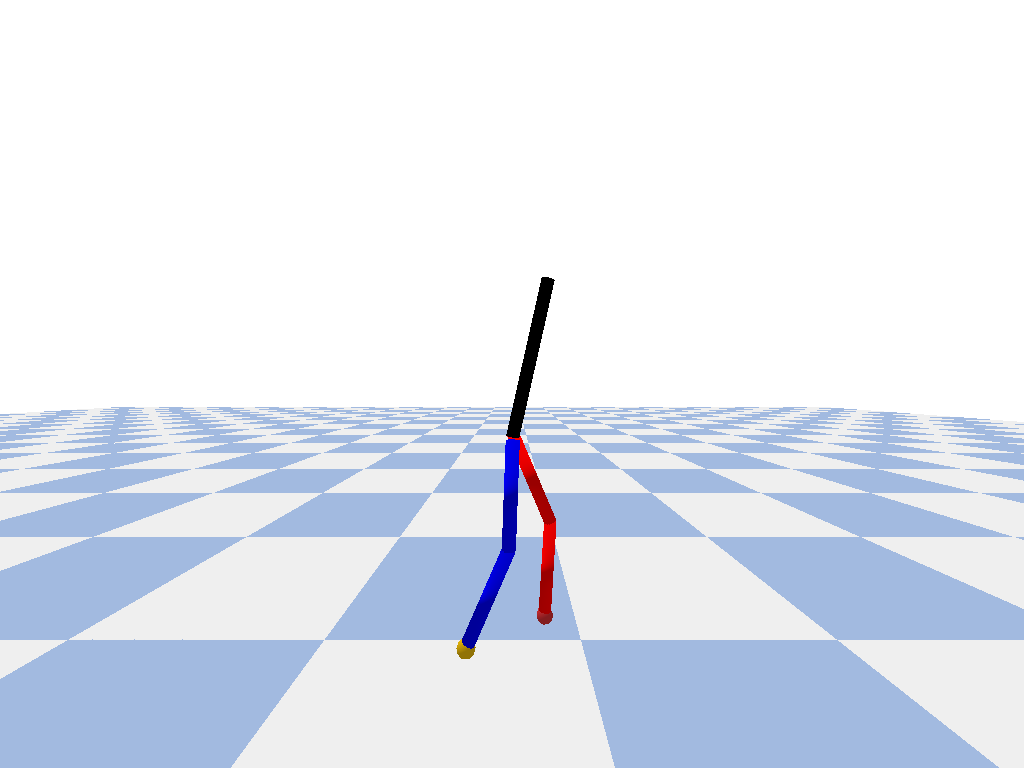}
     \end{subfigure} 
     \hspace{-0.3cm}
     \begin{subfigure}[b]{0.1\textwidth}
         \centering
         \includegraphics[width=\textwidth,trim=16cm 3.5cm 15cm 8cm, clip]{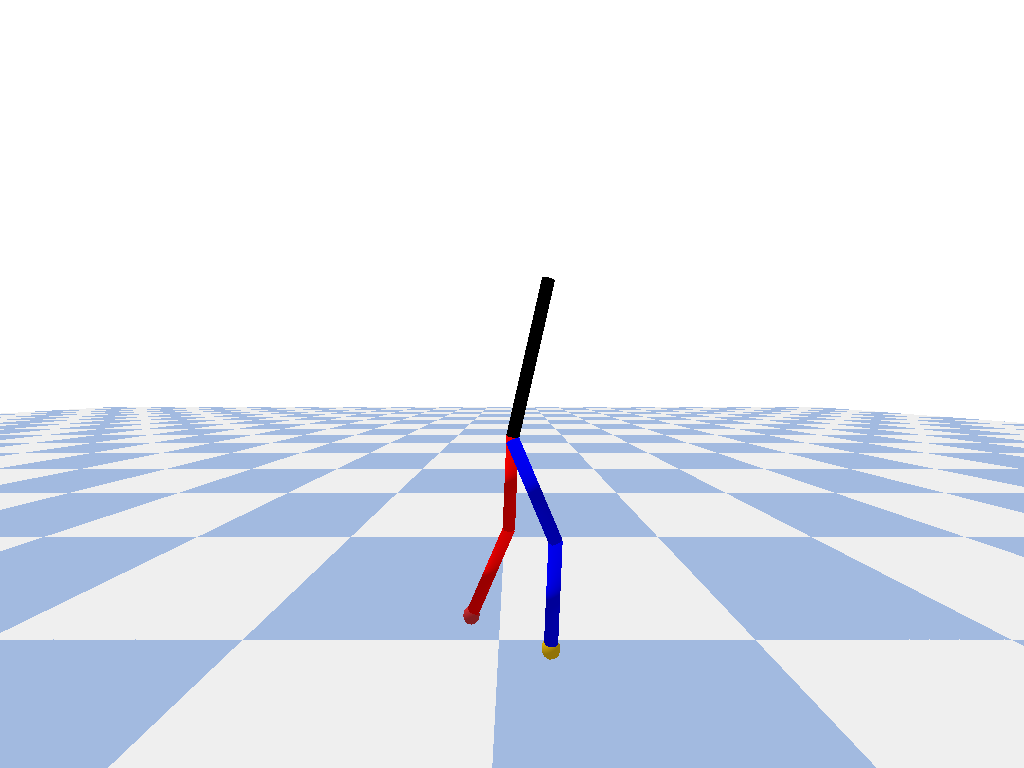}
     \end{subfigure}
     \hspace{-0.3cm}
     \begin{subfigure}[b]{0.1\textwidth}
         \centering
         \includegraphics[width=\textwidth,trim=16cm 3.5cm 15cm 8cm, clip]{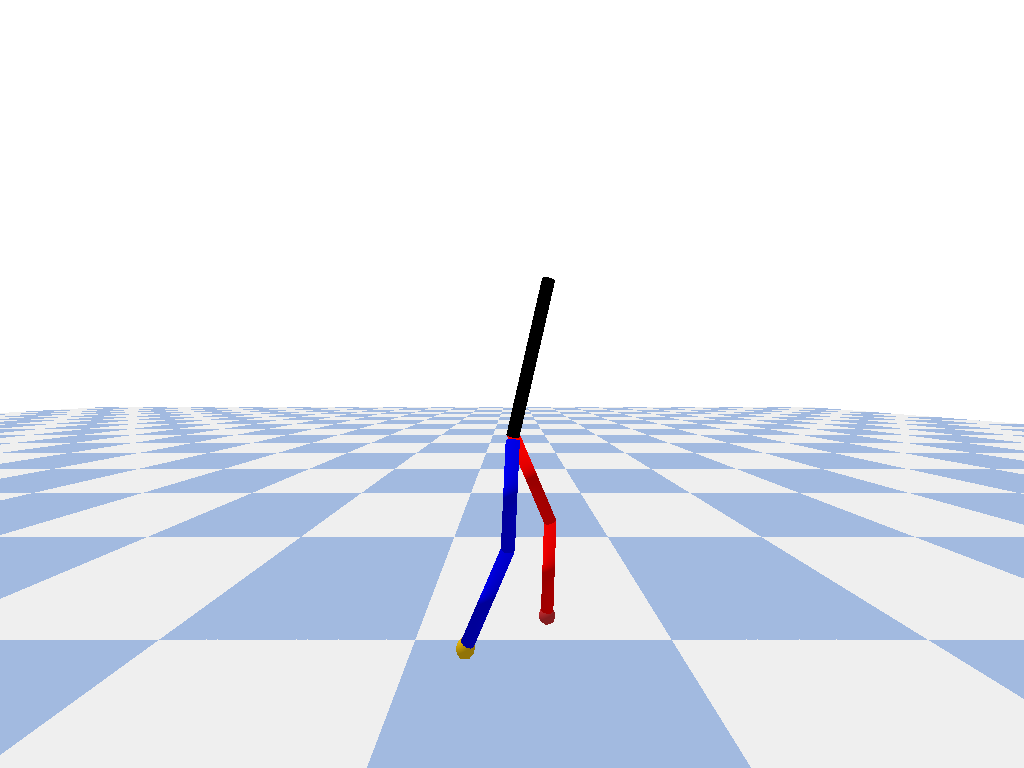}
     \end{subfigure}
     \hspace{-0.3cm}
     \begin{subfigure}[b]{0.1\textwidth}
         \centering
         \includegraphics[width=\textwidth,trim=16cm 3.5cm 15cm 8cm, clip]{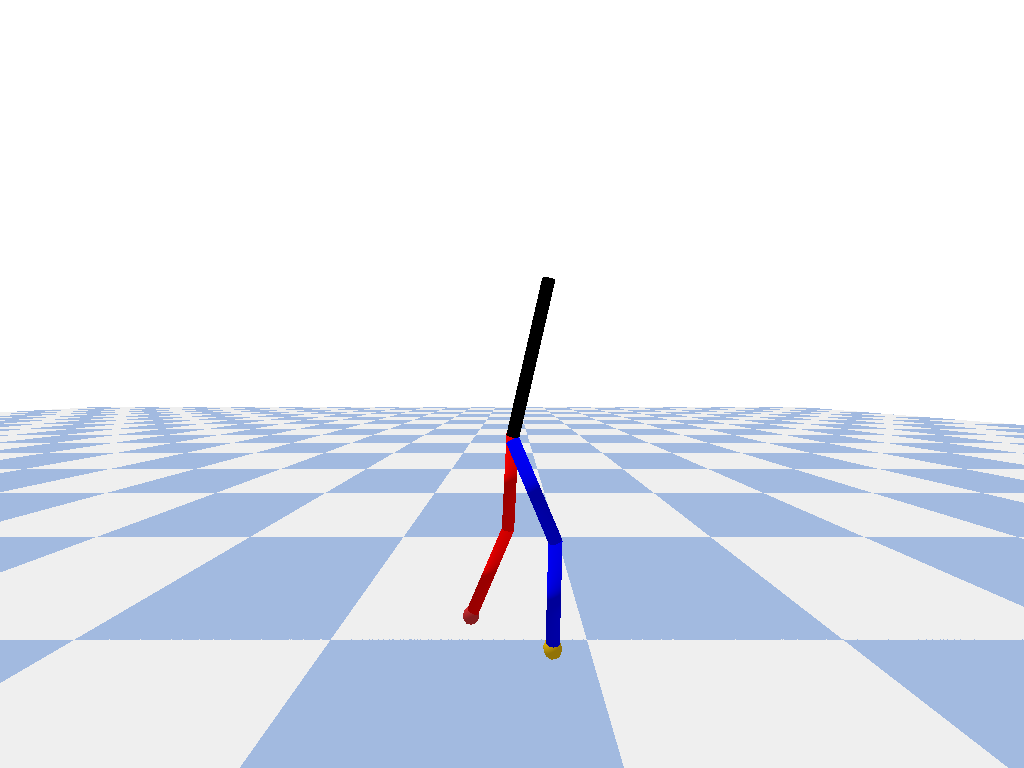}
     \end{subfigure}
     \hspace{-0.3cm}
     \begin{subfigure}[b]{0.1\textwidth}
         \centering
         \includegraphics[width=\textwidth,trim=16cm 3.5cm 15cm 8cm, clip]{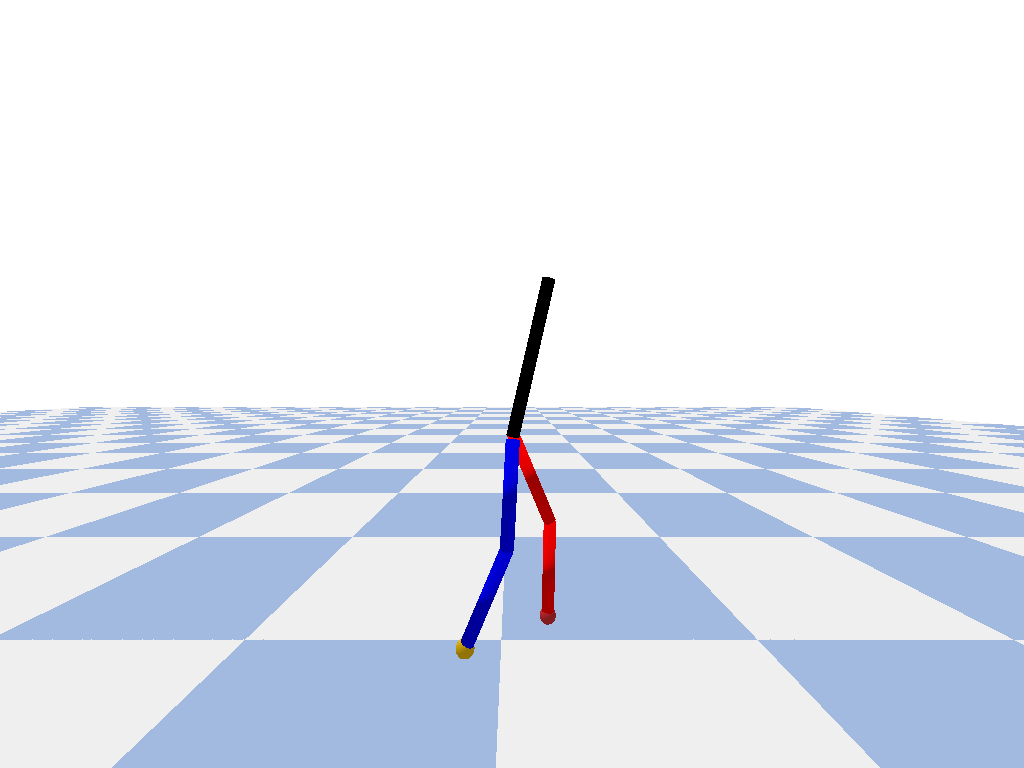}
     \end{subfigure}
     
     \vskip\baselineskip
     \begin{subfigure}[b]{\textwidth}
         \centering
         \includegraphics[width=\textwidth, trim=0.2cm 0cm 0.2cm 0cm, clip]{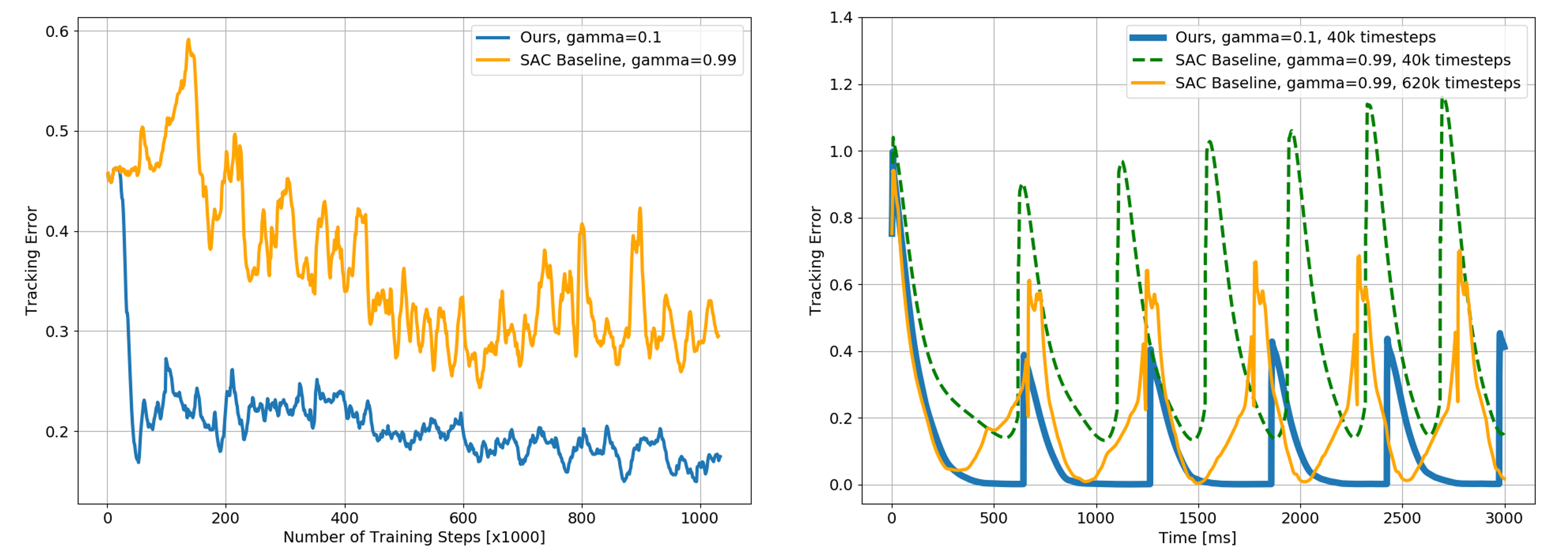}
     \end{subfigure}

     \caption{(Top) Snapshots of RABBIT \cite{chevallereau2003rabbit}, a five-link bipedal robot, successfully walking with our learned controller in the PyBullet simulator \cite{coumans2019}. (Bottom-Left) Average tracking error (lower is better) per episode at different stages of the training process when fine-tuning a model-based walking controller under model mismatch. In blue, using our CLF-based reward formulation and SAC, the robot learns a stable walking gait with only 40k steps (40 seconds) of training data. In orange, with a baseline that uses a typical reward penalizing the tracking error to the target gait, the training takes longer to converge and does not achieve the same performance. The results show the best performance for both method across different discount factors and training hyper-parameters. (Bottom-Right) Comparison of the tracking error of roll-outs of different learned walking policies. In blue, a policy learned with 40k steps of the environment using our CLF-based reward. In dashed green, a policy learned using the baseline reward with 40k steps of the environment. In orange, a policy learned using the baseline reward with 620k steps of the environment (best baseline policy). The jumps in tracking error occur at the swing-leg impact times. The policy learned with our reward formulation clearly outperforms the baseline, even when the baseline has 15 times as much data.}
     \label{fig:rabbit-results}
\end{figure}

In this section, we provide further details on applying our design methodology to fine-tune a model-based walking controller for a bipedal robot. As mentioned in Section \ref{sec:examples}, we first design a CLF around the target gait using the nominal model as in \cite{ames2014rapidly} to be used in our reward formulation. As a benchmark comparison, we also train policies with a typical reward which penalizes the distance to the target motion. For both approaches we use the SAC algorithm to optimize the policy. We plot the best performance we have been able to obtain from each method by sweeping across different discount factors and algorithm hyper-parameters in Figure \ref{fig:rabbit-results}. In particular, the top of this Figure depicts snapshots of the stable walking controller our method obtains after only 40k steps of the environment, which corresponds to only 40 seconds of data given the 1kHz frequency of the controller. The bottom left depicts the average tracking error during the training process for both methods. Finally, the bottom-right plots the tracking error over a few representative rollouts. Note that the tracking error for both methods `jumps' each time one of the feet impacts the ground. These jumps occur when the swing-foot impacts with the ground and are an unavoidable feature of the environment. Thus, in this context a stable walking controller needs to rapidly converge to the target motion over the course of the next step to maintain stability of the walking motion. As the learning curve demonstrates, our approach is able to significantly reduce the average tracking error per episode after only 40k steps of the environment, while the baseline does not reach a similar level of performance even after 1.2 million steps. As the rollouts in the bottom-right demonstrate, our method learns a desirable tracking controller which smoothly decreases the tracking error between each impact event after only 40 thousand steps. In contrast, after 40 thousand steps the baseline controller diverges from the target motion, corresponding to a fall after only a few steps. After 620 thousand steps, the baseline controller is able to maintain the stability of the walking motion, yet the tracking performance is notably worse than our method at 40 thousand steps, despite having access to around 15 times as many samples.

\subsection{Inverted Pendulum Results}
 The states of the system are $x = (\theta,\dot{\theta}) \in \R^2$, where $\theta$ is the angle of the arm from the vertical position, and the input $u \in \R$ is the torque applied to the joint. In each of the reinforcement learning experiments reported in Section \ref{sec:examples} for this system we sample initial conditions over the range $-\pi \leq \theta \leq \pi$ and $ -0.1<\dot{\theta} <0.1$.

 \begin{figure}
    \centering
    \vspace{-1em}
    \includegraphics[width=0.9\textwidth,trim={0 1.3cm 0 0},clip]{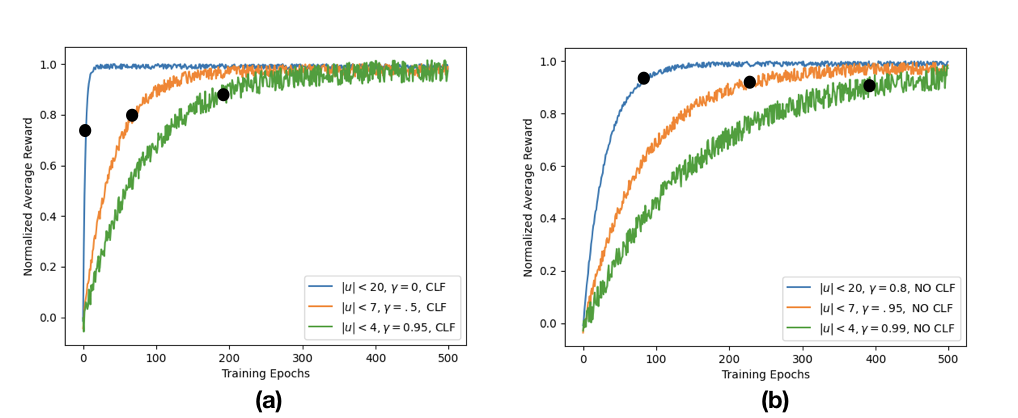}
    \caption{\small Learning curves for an inverted pendulum system under different input constraints. The curves plotted correspond to the smallest discount factors that led to stabilizing policies. On the left, the obtained learning curves use a CLF in the reward. On the right, the reward does not include the CLF term. The black dots denote the first stabilizing policy for each training. For each setting we plot the learning curve for the discount factor that achieved the best performance.  }
    \vspace{-1em}
    \label{fig:inverted pendulum}
\end{figure}
 
We first train a stabilizing controller using a `typical' cost function of the form $r_k = -\|x_k\|_2^2 -0.1\|u_k\|_2^2$, and then train a controller using the reshaped cost $r_k = -\left[W(F(x_k,u_k)) -W(x_k)\right]-\|x_k\|_2^2 - 0.1\|u_k\|_2^2$.  We use the soft actor critic (SAC) algorithm \cite{haarnoja2018learning} and each training epoch consisted of 5 episodes with 100 simulation steps each, where each time step for the simulator is $0.1$ seconds. For both forms of cost function, we sweep across different values of discount factors (from $\gamma =0$ to $\gamma =0.95$ in increments of $0.05$ and also tried $\gamma =0.99$)
to $1)$ determine which values of discount factors lead to stabilizing policies and $2)$ which discount factor allows the agent to learn a stabilizing controller most rapidly. To determine whether a given controller stabilizes the system we randomly sample 20 initial conditions and see if each trajectory reaches the set $\{ x\in \R^n : \|x\|_2<0.05\}$ within 20 seconds of simulation. For each scenario, the smallest discount factor that lead to a stabilizing controller was also the discount factor that cause the agent to learn a stabilizing controller with the least amount of data. 

 Training curves for each of the critical values of the discount factor are depicted in Figure \ref{fig:inverted pendulum} for each of the cost formulations and input constraints. Each curve indicates the average reward per epoch across 10 different training runs and reports the best results for each scenario after an extensive hyper-parameter sweep. We normalize each training curve so that a reward of $0$ indicates the average reward during the first epoch, while a reward of $1$ is the largest average reward obtained across all epochs. On each of the training curves the black dot denotes the first training epoch at which a stabilizing controller was obtained. 
 
As illustrated by the plots in Figure \ref{fig:inverted pendulum} (a), the addition of the CLF enables our method to more rapidly learn a stabilizing controller in each setting and consistently decreases the amount of data that is needed to learn a stabilizing controller, even when $W$ is not a global CLF for the system. However, the effects are more pronounced when the input constraints are less restrictive and $W$ is a better candidate CLF. For example, when $|u| <20$ our approach is able to learn a stabilizing controller in $5$ iterations, whereas it takes $92$ iterations with the original cost (our approach takes $\sim 5.4 \%$ as many samples). Meanwhile when $|u|<4$ our approach takes $198$ iterations while the original cost takes $389$ iterations (our approach takes $\sim 51\%$ as many samples).Moreover, we observe that larger discount factors are required when $|u|\leq7$ and $|u|\leq4$, as $W$ becomes a poorer candidate CLF for these cases.

\end{document}